\documentclass{article}

\usepackage{microtype}
\usepackage{graphicx}
\usepackage{subfigure}
\usepackage{booktabs} %

\usepackage{hyperref}

\usepackage[utf8]{inputenc} %
\usepackage[T1]{fontenc}    %

\usepackage{natbib}
\usepackage{graphicx}
\usepackage{amsthm}
\usepackage{amsmath}
\usepackage{amssymb}
\usepackage{amsopn}
\usepackage{amsbsy}
\usepackage{subfigure}
\usepackage{optidef}

\usepackage{dsfont}

\usepackage{hyperref}       %
\usepackage{url}            %
\usepackage{booktabs}       %
\usepackage{amsfonts}       %
\usepackage{nicefrac}       %
\usepackage{microtype}      %
\usepackage{xcolor}         %
\usepackage{enumitem}

\newcommand{\transp}[1]{#1^\top}

\newcommand{\R}{\mathbb{R}}

\newcommand{\bigo}{\mathcal{O}}

\newcommand{\Pro}{\mathbb{P}}

\newcommand{\E}{\mathbb{E}}

\newcommand{\eps}{\varepsilon}

\newcommand{\Obs}[1]{\mathcal{O}_{#1}}

\newcommand{\gau}[1]{\mathcal{N}(0, #1)}

\newcommand{\one}{\mathds{1}}
\newcommand{\utov}{{u\to v}}
\newcommand{\tmix}{\tau_{\operatorname{mix}}}

\usepackage{mathtools}
\usepackage[linesnumbered, ruled]{algorithm2e}
\usepackage[capitalize,noabbrev]{cleveref}

\DeclarePairedDelimiter\abs{\lvert}{\rvert}%
\DeclarePairedDelimiter\norm{\lVert}{\rVert}
\newtheoremstyle{lemme}
  {\topsep}{\topsep}%
  {}{}%
  {\bfseries}{}%
  {  }{}%
\theoremstyle{lemme}

\newtheorem{thm}{Theorem}
\newtheorem{lemma}{Lemma}

\newtheorem{defn}{Definition}
\newtheorem{prop}{Proposition}
\newtheorem{rmk}{Remark}

\newtheorem{assum}{Assumption}

\newcommand{\paren}[1]{\left( #1 \right)}

\newcommand{\selfloop}{\kappa}
\usepackage[accepted]{icml2024}

\usepackage{amsmath}
\usepackage{amssymb}
\usepackage{mathtools}
\usepackage{amsthm}
\usepackage{placeins}
\usepackage[capitalize,noabbrev]{cleveref}
\usepackage[normalem]{ulem}

\theoremstyle{plain}

\theoremstyle{definition}

\theoremstyle{remark}

\usepackage[textsize=tiny]{todonotes}

\icmltitlerunning{Differentially Private Decentralized Learning with Random Walks}

\begin{document}

\twocolumn[
\icmltitle{Differentially Private Decentralized Learning with Random Walks}

\icmlsetsymbol{equal}{*}

\begin{icmlauthorlist}
\icmlauthor{Edwige Cyffers}{yyy}
\icmlauthor{Aurélien Bellet}{zzz}
\icmlauthor{Jalaj Upadhyay}{comp}
\end{icmlauthorlist}

\icmlaffiliation{yyy}{Université de Lille, Inria, CNRS,
Centrale Lille, UMR 9189 - CRIStAL,
F-59000 Lille, France}
\icmlaffiliation{comp}{Rutgers University}
\icmlaffiliation{zzz}{Inria, Univ Montpellier, Montpellier, France}

\icmlcorrespondingauthor{Edwige Cyffers}{edwige.cyffers@inria.fr}

\icmlkeywords{Machine Learning, ICML}

\vskip 0.3in
]

\printAffiliationsAndNotice{}  %

\begin{abstract}
The popularity of federated learning comes from the possibility of better scalability and the ability for participants to keep control of their data, improving data security and sovereignty. Unfortunately, sharing model updates also creates a new privacy attack surface. In this work, we characterize the privacy guarantees of decentralized learning with random walk algorithms, where a model is updated by traveling from one node to another along the edges of a communication graph. 
Using a recent variant of differential privacy tailored to the study of decentralized algorithms, namely Pairwise Network Differential Privacy, we derive 
closed-form expressions for the privacy loss between each pair of nodes where the impact of the communication topology is captured by graph theoretic quantities. Our results further reveal that random walk algorithms yield better privacy guarantees than gossip algorithms for nodes close to each other. We supplement our theoretical results with empirical evaluation of synthetic and real-world graphs and datasets.
\end{abstract}

\section{Introduction}

Federated learning allows multiple data owners to collaboratively train a model without sharing their data~\citep{kairouz_advances_2019}. Some federated algorithms rely on a central server to orchestrate the process and aggregate model updates~\citep{mcmahan2017fl}, with the downsides of creating a single point of failure and limiting the scalability in the number of participants~\citep{Lian2017b}. In this work, we focus on fully decentralized algorithms that replace the central server by peer-to-peer communications between participants, viewed as nodes in a sparse network graph~\cite{pmlr-v119-koloskova20a, NEURIPS2021_5f25fbe1, Lian2017b,pmlr-v206-le-bars23a,nedic2018network,pmlr-v80-tang18a}. In addition to their scalability, these algorithms can exploit the natural graph structure in some applications, such as social networks where users are linked to their friends, or the geographical position of devices that induces faster communications with the closest users.

Keeping data decentralized can reduce communication costs and latency, and it is also welcomed from a privacy perspective when the data contains personal information or represents a crucial asset for businesses. However, sharing model parameters can indirectly leak sensitive information and allow data reconstruction attacks~\citep[see e.g.][]{fedLearningAttacks,nasr2018comprehensive,shokri2017membership,zhu19deep}. To mitigate this problem, {\em differential privacy} (DP)~\cite{dwork2006calibrating} has become the \emph{de facto} standard in machine learning to provide robust guarantees of privacy. In a nutshell, DP compares two learning scenarios that only differ from the data of a single user and ensures that the output distribution of the algorithm remains similar. This, in particular, means that no attacker can learn too much about the private data of a user by inspecting the outputs, even if they have access to arbitrary auxiliary information.

Decentralized learning algorithms can be made differentially private by having each node add noise to their model updates before sharing them with their neighbors in the graph~\citep{Bellet2018a, leasgd,Huang2015a, 9524471,admm}. An important challenge is then to bound as tightly as possible the privacy leakage based on the level of noise and the threat model considered to achieve the best possible privacy-utility trade-off. The baseline approach relies on {\em local DP}, which assumes that all information a node sends is observed by all other nodes.
This leads to overly pessimistic privacy guarantees for decentralized algorithms because nodes only observe the messages sent by their direct neighbors.

Recent work has shown that it is possible to leverage the graph topology of decentralized algorithms to develop more tailored privacy guarantees specific to the relation between the different nodes based on the notion of \emph{pairwise network differential privacy} (PNDP) \cite{Cyffers2020PrivacyAB,muffliato}.
PNDP takes into account the fact that each node only has a local view of the communications, and allows to reason on the privacy leakage between each pair of nodes based on these local views. 
Intuitively, if two nodes are farther apart, the privacy leakage should depend on their relative position in the graph, which
matches the natural setting where edges come with some trust level, such as in social networks where edges indicate relationships.
\citet{muffliato} showed that this intuition is correct for {\em gossip algorithms}, where all nodes update their current model and mix it with their neighbors at every step. However, a major drawback of gossip algorithms is that, even in their more asynchronous versions, they generate redundant communications and require all nodes to be largely available (since any node can be updated at any time). Redundant communication and availability has been touted as a major obstacle in distributed private learning \citep{smith2017interaction}.

\noindent \textbf{Our Contributions.} 
In this work, we study the privacy guarantees of random walk algorithms \cite{lopes2007incremental,rism,walkman,Cyffers2020PrivacyAB, even2023stochastic}, a popular alternative to gossip in the fully decentralized setting. In these algorithms, a \emph{token} holding the model's current state is updated by a node at each time step and then forwarded to a random neighbor. Random walk algorithms do not require global synchronization as the token is sent as soon as the current node finishes its update. They can also easily cope with temporary unavailability, and are known to be fast in practice. %

We first introduce a private version of decentralized {\em stochastic gradient descent} (SGD) based on random walks on arbitrary graphs: in a nutshell, the node holding the model at a given step updates it with a local SGD step, adds Gaussian noise and forwards it to one of its neighbor chosen with appropriate probability. Focusing on the strongly convex setting, we then establish a convergence rate for our algorithm by building upon recent results on SGD under Markovian sampling \cite{even2023stochastic}, and show that the result compares favorably to its gossip SGD counterpart. Our main contribution lies in precisely characterizing the privacy loss between all pairs of nodes using a PNDP analysis. We obtain elegant closed-form expressions that hold for arbitrary graphs, capturing the effect of a particular choice of graph through graph-theoretic quantities. We also show how our general closed-form expression yields explicit and interpretable results for specific graphs. Finally, we use synthetic and real graphs and datasets to illustrate our theoretical results and show their practical relevance compared to the gossip algorithms analyzed in previous work.

In summary, our contributions are as follows:
\begin{enumerate}
    \item We propose a private version of random walk stochastic gradient descent for arbitrary graphs (\Cref{algo:rw_gd});
    \item We establish its convergence rate for strongly convex loss functions (\Cref{thm:optim});
    \item We derive closed-form expressions for the privacy loss between each pair of nodes that capture the effect of the topology by graph-theoretic quantities (\Cref{thm:privacy});
    \item We theoretically and experimentally compare our guarantees to those of gossip algorithms, highlighting the superiority of our approach in several regimes.
\end{enumerate}

\section{Related Work}

\paragraph{Random walks for decentralized optimization.} Optimizing the sum of local objective functions by having a token walk on the graph has a long history in the optimization community \cite{rism,lopes2007incremental, walkman}. The main difficulty is to handle the bias introduced by the sampling of the nodes, as a random walk locally forces a structure that differs from the stationary distribution \cite{Sun2018OnMC}. One way to avoid this bias is to restrict the underlying graph structure to be the complete graph~\cite{Cyffers2020PrivacyAB,pmlr-v202-cyffers23a} or to perform an update only after several steps on the walk \cite{hendrikx2022tokens}, but this comes at a high communication cost. In this work, we rely on a recent proof that casts Markov chain updates as a specific case of stochastic gradient descent with delays in order to get rid of the dependency of the Markov sampling by waiting a sufficient number of steps for the analysis \cite{even2023stochastic}.

\noindent \textbf{Private decentralized optimization.} A classic line of work to improve privacy in a decentralized setting aims to prove that nodes cannot access enough information to reconstruct exactly the contribution of a given node \cite{Gupta2018InformationTheoreticPI,dpmano}. However, these approaches do not provide robust guarantees against approximate reconstruction attacks or adversaries with auxiliary knowledge. Another direction is to rely on local DP \citep{Bellet2018a, leasgd, Huang2015a,9524471,admm}, where each node assumes that everything they share is public, but this comes at a high cost for utility \cite{Chan2012,Wang2018b,Zheng2017}. While it is possible to mitigate this drawback by using other schemes such as shuffling or secure aggregation \cite{Bonawitz2017a,cheu2019distributed,clones,Liew2022NetworkSP}, it requires additional computation, communication as well as an overhaul of the system architecture, making it very difficult to deploy in practice. These limitations have motivated the development of intermediary trust models specific to fully decentralized settings.

\noindent \textbf{Network differential privacy.} 
\citet{Bellet2020a} were the first to suggest that decentralization can amplify privacy. This was made precise by \citet{Cyffers2020PrivacyAB}, who introduced Network Differential Privacy (NDP) and proved better privacy guarantees for the complete graph and the ring graph. For the complete graph, \citet{pmlr-v202-cyffers23a} further used NDP to analyze the privacy guarantees of decentralized ADMM. The case of the ring graph was further studied by \citet{stragglers}, taking into account the additional complexity of dealing with straggler nodes that slow down the computation.
In this work, we rely on the pairwise NDP variant introduced in \citet{muffliato}, where the authors study private gossip algorithms on arbitrary graphs. We provide both theoretical and empirical comparisons to these prior results, showing significant advantages in favor of our random walk algorithm.

\section{Preliminaries}
\begin{defn}
[Irreducibility]
A Markov chain defined by transition matrix $W \in \mathbb R^{n \times n}$ is called {\em irreducible} if for any two states $i,j$, there exists an integer $t$ such that 
\[
Pr[X_t =j | X_0=i] >0.
\]
\end{defn}

\begin{defn}
    [Aperiodicity]
    The period of a state $i$ is defined as the greatest common divisor of the set of natural numbers, $\{ t: Pr[X_t=i | X_0=i]>0 \}$. A state $i$ is called {\em aperiodic} if its period equals $1$. A Markov chain is {\em aperiodic} if all its states are aperiodic. 
\end{defn}

Irreducibility and aperiodicity ensure that, after enough steps, the probability of transit from any state to any other state is positive. 

\begin{defn}
    [Stationary distribution]
    For a Markov chain defined by the transition matrix $W \in \mathbb R^{n \times n}$, a probability distribution $\pi$ is a stationary distribution if $P\pi = \pi$. 
\end{defn}

\section{Problem Setting and Background}

We consider a set $V=\{1,\dots,n\}$ of users (nodes), each user $v$ holding a local dataset $\mathcal{D}_v$. The nodes aim to privately optimize a separable function over the joint data $\mathcal{D} = \cup_{v \in V} \mathcal{D}_v$:
\begin{equation}
f(x) = \sum_{v=1}^n \pi_v f_v(x),
    \label{eq:obj}
\end{equation}
where $x\in\mathbb{R}^d$ represents the parameters of the model and the local function $f_v$ depends only on the local dataset of node $v$, and $\pi_v\geq 0$ is the weight given to $f_v$ (in practice, the vector $\pi$ will correspond to the stationary distribution of the random walk as defined below).
Below, we introduce the notions and assumptions related to random walks and precisely define the privacy threat model we consider.

\subsection{Random Walks}
We consider that the underlying network is represented by a connected graph $G=(V,E)$. Two nodes $u$ and $v$ are neighbors when there is an edge $(u,v) \in E$, which indicates that $u$ and $v$ can communicate.
Our random walk algorithm will involve a token following a Markov chain on this graph, where the probability of taking each edge is given by an $n\times n$ transition matrix $W$: if the token is in $u$ at step $t$, $v$ receives the token at time $t+1$ with probability $W_{uv}$.

\begin{defn}[Transition matrix]
    A {\em transition matrix} $W$ on graph $G=(V,E)$ is a stochastic matrix, 
    i.e., $\forall u \in V$, $\sum_{v \in V} W_{uv}=1$, which satisfies $(u,v)\notin E \Rightarrow W_{uv}=0$.
\end{defn}

In particular, the probability for the token to go from node $u$ to node $v$ in $k$ steps is given by the $k$-th power of the transition matrix $W_{uv}^k$. To derive convergence in optimization, we need standard assumptions on this transition matrix, which ensure that the Markov chain behaves similarly to a fixed distribution after enough iterations.

\begin{assum}
\label{assum:optim}
    The transition matrix is {\em aperiodic} and {\em irreducible}, that is, there exists a time $t_0$ such that for all $t \geq t_0$ and any pair of vertices $u$ and $v$, $W_{uv}^t > 0$, i.e., the token can go from $u$ to $v$ in $t$ steps.
\end{assum}

Under this assumption, the Markov chain has a \emph{stationary distribution} $\pi$ (belonging to the $n$-dimensional simplex) such that $\pi = \pi W$, and the convergence speed is governed by the mixing time \cite{levin2017markov}.

\begin{defn}[Mixing time]
    The mixing time $\tmix(\iota)$ of a Markov chain of transition matrix $W$ is the time needed for the walk to be close to a factor $\iota$ of its asymptotic behavior:
    \begin{equation}
        \tmix(\iota ) = \min \left(t: \max_v \norm{(W^t)v - \pi}_{TV} \leq \iota \right)\,,
    \end{equation}
    where $\| P-Q\|_{TV}$ is the {\em total variation distance} between two probability measures $P$ and $Q$ defined over the same measurable space $(\Omega, \mathcal F)$:
    $
    \|P-Q\|_{TV} = \sup_{A \in \mathcal F} |P(A)-Q(A)|.
    $
\end{defn}

We sometimes omit $\iota$ and write $\tmix:= \tmix(1/4)$. From this quantity, it is easy to derive the mixing time for an arbitrary $\iota$ from the following equation: $\tmix(\iota) \leq \lceil \log_2 (1/\iota) \rceil \tmix$  \cite{levin2017markov}. The mixing time depends on the spectral gap of the graph, which is the difference between the largest eigenvalue of the transition matrix $W$ (which is always equal to one and associated with $\pi$) and the absolute value of the second-largest eigenvalue. We denote this quantity by $\lambda_W$, from which we get the following bound on the mixing time:
$\tmix \leq \lambda_W^{-1} \ln (1/\min_v \pi_v)$.

A natural choice of the transition matrix is to give the same importance to every node and to assume symmetry in the weight of the communications.

\begin{assum}
\label{assum:privacy}
    The transition matrix $W$ is bi-stochastic and symmetric.
\end{assum}    

Under this assumption, we have $\pi = \one/n$, where $\one$ is an all-one vector. As $W$ is symmetric, it can be decomposed using the spectral theorem, and all the eigenvalues are real. For any connected graph, it is possible to construct a transition matrix that satisfies this assumption, for instance by using the Hamilton weighting on the graph where transitions are chosen uniformly among the neighbors. Let $d_u$ denote the degree of node $u$. Then 
\begin{align*}
    W_{uv} &= \Pro(X^{t+1} = i | X^t = j) = \frac{1}{\max\{d_u, d_v\}}, \\
    W_{uu} &= 1-\sum_{v\neq u} W_{uv}.
    \label{eq:W_uu}
\end{align*}

\begin{rmk}
Our optimization results of Section~\ref{sec:optim} will only require Assumption~\ref{assum:optim} and thus cover \eqref{eq:obj} with non-uniform weights. Conversely, our privacy results of Section~\ref{sec:privacy} will only require Assumption~\ref{assum:privacy}.
\end{rmk}

\subsection{Privacy Threat Model}
In this paper, we aim to quantify how much information each node $u\in V$ leaks about its local dataset $\mathcal{D}_u$ to any other node $v$ during the execution of a decentralized learning algorithm. We assume nodes to be {\em honest-but-curious} (i.e., they faithfully follow the protocol) and non-colluding (see
\cref{app:collusion} for the possibility of collusion and how it can be seen as a modification of the graph).

We consider the graph $G$ and the transition matrix $W$ to be known by all nodes. This scenario is quite common; for instance, in the healthcare domain (where nodes represent hospitals and collaboration between hospitals is public knowledge) or in social networks. We will measure the privacy leakage using Differential Privacy (DP), and more precisely, a variant known as the Pairwise Network Differential Privacy (PNDP) \citep{Cyffers2020PrivacyAB,muffliato} tailored to the analysis of decentralized algorithms. In the rest of this section, we introduce the relevant definitions and tools and formally define what is observed by each node during the execution of a random walk algorithm.

\looseness=-1 \textbf{Differential Privacy.} DP \cite{dwork2013Algorithmic} quantifies the privacy loss incurred by a randomized algorithm $\mathcal{A}$ by comparing its output distribution on two \emph{adjacent} datasets $\mathcal{D}$ and $\mathcal{D}'$. The guarantee depends thus on the granularity chosen for the adjacency relation. In this work, we adopt user-level DP, where $\mathcal{D} = \cup_{v\in V} \mathcal{D}_v$ and $\mathcal{D}' = \cup_{v\in V}  \mathcal{D}_v'$ are adjacent datasets, denoted by $\mathcal{D}\sim \mathcal{D}'$,  if there exists at most one user $v \in V$ such that $\mathcal{D}_v\neq \mathcal{D}_v'$. We further denote $\mathcal{D}\sim_v \mathcal{D}'$ if $\mathcal{D}$ and $\mathcal{D}'$ differ only in the local dataset of user $v$.

We use Rényi Differential Privacy (RDP) to measure the privacy loss, due to its theoretical convenience and better composition properties than the classical $(\epsilon,\delta)$-DP. We recall that any $(\alpha, \eps)$-RDP algorithm is also $(\eps+\ln(1/\delta)/(\alpha-1),\delta)$-DP for any $0<\delta<1$ \cite{DBLP:journals/corr/Mironov17}.

\begin{defn}[Rényi Differential Privacy]
  \label{def:DP}
  An algorithm $\mathcal{A}$ satisfies $(\alpha, \eps)$-Rényi Differential Privacy (RDP) for $\alpha>1$ and $\eps>0$ if for all pairs of neighboring datasets $\mathcal{D} \sim \mathcal{D}'$:
  \begin{equation}
  \label{eq:DP}
  D_{\alpha} \left(A(\mathcal{D}) || A(\mathcal{D}') \right) \leq \eps\,,
  \end{equation}
  where $D_{\alpha}(X||Y)$ is the \emph{Rényi divergence} between the random variables $X$ and $Y$:
  \begin{equation*}
 D_{\alpha}(X||Y)=\frac{1}{\alpha-1}\ln \int \left(\frac{\mu_{X}(z)}{\mu_Y(z)}  \right)^{\alpha} \mu_Y (z) dz \,.
  \end{equation*}
with $\mu_X$ and $\mu_Y$ the respective densities of $X$ and $Y$.
\end{defn}

The Gaussian mechanism ensures RDP by adding Gaussian noise to the output of a non-private function $g$, i.e., $\mathcal{A}(\mathcal{D}) = g(\mathcal{D}) + \eta$ with $\eta \sim \gau{ \sigma^2}$ satisfies $(\alpha, \alpha \Delta_g^2/2 \sigma^2)$-RDP for any $\alpha >1$, where $\Delta_g=\sup_{\mathcal{D}\sim \mathcal{D}'}\|g(\mathcal{D})-g(\mathcal{D}')\|_2$ is the sensitivity of $g$ ~\cite{DBLP:journals/corr/Mironov17}.

This baseline privacy guarantee can be amplified when the result is not directly observed but instead used for subsequent computations. In particular, when considering the consecutive applications of {\em non-expansive} (i.e. $1$-Lipschitz) operators, we can rely on the so-called \emph{privacy amplification by iteration} effect that we will leverage in our analysis.

\begin{thm}[Privacy amplification by iteration, \citealp{ampbyiteration}]
\label{thm:amp}
    Let $T^{1}, \dots, T^{K}, T'^{1}, \dots, T'^{K}$ be non-expansive
    operators, an initial random state $x^{0}\in U$ , and $(\zeta^{k})_
    {k=1}^K$ a
    sequence
    of noise
    distributions. Consider the noisy iterations $x^{k+1}=T^{k+1}(x^k)+\eta^
    {k+1}$ and
    $\bar x^{k+1}=T^{k+1}(\bar x^k)+ \bar \eta^{k+1}$ where $\eta^k$ and $\bar \eta^k$ are
    drawn independently from distribution $\zeta^{k+1}$.
    Let $s_k =
    \sup_{x\in U} \norm{T^k (x) - \bar T^k(x)}$. Let $(a_k)_{k=1}^K$ be a
    sequence of
    real
    numbers such that
    \begin{equation}
    \forall k \leq K, \sum_{i \leq k} s_i \geq \sum_{i \leq k} a_i, \text{
    and } \sum_{i \leq K} s_i = \sum_{i \leq K} a_i\,.
    \end{equation}
    Then,
    \begin{equation}
        D_{\alpha}(x^K || \bar x^K) \leq \sum_{k=1}^K \sup _{x:\|x\| \leq a_k} D_
        {\alpha}(\zeta_k * \mathbf{x} \| \zeta_k)\,,
    \end{equation}  
    where $*$ is the convolution of probability distributions and $
    \mathbf{x}$ denotes the distribution of the random variable that is always
    equal to $x$.
\end{thm}

\textbf{Pairwise Network Differential Privacy.} PNDP allows us to capture the limited view that nodes have in decentralized algorithms and to model privacy guarantees specific to each pair of nodes \citep{Cyffers2020PrivacyAB,muffliato}. Below, the view of a user $v$ is denoted by $\Obs{v}\big(\mathcal{A}(\mathcal{D})\big)$.
\begin{defn}[Pairwise Network DP]
    \label{def:indiv_ndp}
    For $b: V \times V \rightarrow \R^+$, an algorithm $\mathcal{A}$ satisfies $(\alpha, b)$-Pairwise Network DP (PNDP) if for all pairs of distinct users $u, v \in V$ and neighboring datasets $\mathcal{D} \sim_u \mathcal{D}'$:%
    \begin{equation}
    \label{eq:network-pdp}
    D_{\alpha}(\Obs{v}(\mathcal{A}(\mathcal{D}))||\Obs{v}(\mathcal{A}(\mathcal{D}'))) \leq b(u,v)\,.
    \end{equation}
   We denote by $\eps_{\utov} = b(u,v)$ the privacy loss from $u$ to $v$ and say that $u$ is $(\alpha, \eps_\utov)$-PNDP with respect to $v$ when inequality~\eqref{eq:network-pdp} holds for $b(u,v)=\eps_\utov$.
\end{defn}

For the random walk algorithms we will consider, the complete output $\mathcal{A}(\mathcal{D})$ consists of the trajectory of the token and its successive values during training. At a given step, the token of the random walk shares its current value only with its current location, but the other nodes cannot see this state. Thus, we define the view of a node $v$ as
\begin{equation}
\begin{aligned}
        \Obs{v}\big(\mathcal{A}(\mathcal{D})\big) = \{&(t, x_t, w) : 
        \text{ the token $x_t$ was in } v \\ & \text{ at time } t 
        \text{ and then sent to } w\}\,.
    \end{aligned}
\label{eq:view_main}
\end{equation}
\looseness=-1 In this definition, nodes know to whom they send the token, but not from whom they receive it. Ensuring the anonymity of the sender can be achieved by using mix networks~\citep{mixnet} or anonymous routing~\citep{tor}. %
However, our results directly extend to the case where the sender's anonymity cannot be ensured, see Remark~\ref{rem:no_anon} in Section~\ref{sec:privacy}.

\vspace{-3mm}
\section{Private SGD with Random Walks}
\label{sec:optim}

In this section, we introduce a decentralized {\em stochastic gradient descent} (SGD) random walk algorithm to privately approximate the minimizer of \eqref{eq:obj}, and analyze its convergence in the strongly convex case. This algorithm, presented in Algorithm~\ref{algo:rw_gd}, generalizes the private random walk algorithm on the complete graph, introduced and analyzed by~\citet{Cyffers2020PrivacyAB}, to arbitrary graphs. Differential privacy is achieved by adding Gaussian noise to the local gradient at each step. The step size is constant over time as commonly done in (centralized) differentially private stochastic gradient descent (DP-SGD) \cite{bassily2014Private}.

\begin{algorithm}[t]
    \DontPrintSemicolon
    \KwIn{transition matrix $W$ on a graph $G$, number of iterations $T$, noise variance $\sigma^2$, starting node $v_0$, initial token value $x_0$, step size $\gamma$, gradient sensitivity $\Delta$, local loss function $f_v$
    }
    \For{$t = 0$ to $T-1$}{
        Draw $\eta \sim \gau{\Delta^2\sigma^2}$\;
        Compute $g_t$ s.t. $\E[g_t] = \nabla f_{v_t}(x_t)$\;
        $x_{t+1} \leftarrow x_t - \gamma (g_t+\eta)$\;
        Draw $u \sim W_{v_t}$ in the set of neighbors of $v_t$\;
        Send token to $u$\;
        $v_{t+1} \leftarrow u$\;

    }
    \caption{\looseness=-1\textsc{Private random walk gradient descent (RW DP-SGD)}}
    \label{algo:rw_gd}
\end{algorithm}

The non-private version of this algorithm converges in various settings under~\cref{assum:optim}. In this work, we adapt a recent proof for the non-private version~\citep{even2023stochastic}. For simplicity, we focus on strongly convex and smooth objectives with bounded gradients at the global optimum.

\begin{assum}[Bounded gradient and strong convexity] We assume that $f$ is $\mu$-strongly convex and $L$-smooth.  Let $x^*$ be its minimizer. We assume that, for $\zeta_*\geq 0$,
    \label{ass:bg}
$\forall v \in V, \norm{\nabla f_v(x^*)}^2 \leq \zeta_*^2\,.
$ %
\end{assum}

In the case of stochastic gradient descent, stochasticity also comes from the fact that we sample from the local dataset. To handle both cases, we define $g_t$ as an unbiased estimator of $\nabla f_{v_t}(x_t)$. We thus require to bound the variance of this estimator.
\begin{assum}[Bounded local noise]
    \label{ass:noise} We assume that the stochastic gradients respect the following condition: 
    $    \mathbb{E}\left[\left\| g_t  - \nabla f_{v_t}\left(x_t\right)\right\|^2 \left. \right\vert x_t, v_t\right] \leqslant \sigma_{sgd}^2.
    $ %
\end{assum}

\begin{thm}
    \label{thm:optim}
    Under Assumptions \ref{assum:optim}, \ref{ass:bg}, and \ref{ass:noise}, for step size $\gamma=\min \left(\frac{1}{L}, \frac{1}{T \mu} \ln \left(T \frac{\left\|x_0-x^{\star}\right\|^2}{\frac{39 L}{\mu^2} \tau_{\operatorname{mix}} \zeta_{\star}^2}\right)\right)$ the iterates verify:
\begin{equation*}
\begin{split}
    &\E(\norm{x_T - x^*}^2) \leq  \, 2 e^{-\frac{T \mu}{L}} \norm{x_0 -x^*}^2 \\
    &  + \bigg(\frac{39 \tmix \zeta_*^2L}{\mu^3T} + \frac{(d\sigma^2\Delta^2+ \sigma_{sgd}^2)L}{\mu^2 T} \bigg)  \ln \frac{T \mu^2 \norm{x_0 -x^*}^2}{ 39L\tmix \zeta_*^2}\,.
\end{split}
\end{equation*}

    \end{thm}
    
    \begin{proof}
    We adapt the proof from \citet{even2023stochastic} that keeps track of the shift due to the Markov sampling of the random walk thanks to a comparison with delayed gradient. Following the same bounding steps and taking expectation over the noise distribution we obtain a similar inequality on the iterates than the non-private version up to an additional term for the noise. Setting the step size to balance the terms leads to the final equality. See \cref{app:optim} for a full proof.
\end{proof}

The convergence rate in \cref{thm:optim} has three terms: the exponential convergence towards the minimizer parameterized by the condition number $L/\mu$ of $f$, the impact of the stochasticity of the random walk in $\tilde{\mathcal{O}}(\tmix \zeta_*^2 L/\mu^3 T)$ and the additional term due to the noise injection in $\tilde{\mathcal{O}}(\sigma^2\Delta^2 L/ \mu^2 T)$. A similar term would appear when adapting the non-private convergence proof for other settings such as under the Polyak-Lojasiewicz condition. 

\textbf{Comparison with private gossip~\cite{muffliato}.} In our random walk algorithm, each step involves computing the gradient of a single node and a single message, whereas the private gossip SGD algorithm of \citet{muffliato} alternates between the computation of local gradients at all nodes in parallel and a multi-step gossip communication phase until (approximate) consensus. Rephrasing the result of~\citet{muffliato} in our notation, each communication phase in \citet{muffliato} requires $\tilde{\mathcal{O}}(\tmix\cdot \ln(\zeta_*^2/\sigma ^2))$ steps where all the nodes send updates synchronously. For the optimization part, the first term is the same in $\bigo(e^{-\frac{T \mu}{L}} \norm{x_0 -x^*}^2)$, but there is only one other term in $\bigo(\sigma^2L/n\mu T)$. Hence, we lose a factor of $n$, but the reduced communication compensates for this. Our analysis is tighter in the sense that we can separate the privacy noise that is independent of the Markov chain, thereby improving the rate of the spectral gap factor compared to a naive analysis. In contrast, the private gossip analysis casts this noise as gradient heterogeneity, because it re-uses the non-private convergence analysis of~\citet{pmlr-v119-koloskova20a}.

\looseness=-1 \textbf{Special case of averaging.} One can use the above algorithm to privately compute the average of values at each node. In this case, we assume that each node has a private value $y_v$ (a float or a vector) and define the local objective function as: 
\begin{equation}
    f_v(x) = \norm{x -y_v}^2\,.
\end{equation}
Note that, in this case, we have $L=\mu=2$. A natural approach is to compute the running average of the visited values with noise injection at each step, which corresponds to Algorithm~\ref{algo:rw_gd} with a decreasing step-size $\gamma_t = 1/t$. The drawback of this approach is that damping the first terms of the sum (when the Markov chain is not yet well mixed) requires a lot of steps and results in slow convergence (at least $\tmix$ steps). Adopting a constant step-size instead, as in Algorithm~\ref{algo:rw_gd}, does not modify the convergence leading terms that stay in $\bigo(1/T)$ for an adequate step-size.

One way to completely remove the influence of the first terms is to have a {\it burn-in} phase, when the token walks without performing any update, to come closer to the stationary distribution. Then, after $\tmix(\iota/2)$ steps, a running average of $4 \delta^2/(\iota^3 \lambda_W)$ is obtained, as proven in Theorem 12.21 of \citet{levin2017markov}.

\section{Privacy Analysis}
\label{sec:privacy}

In this section, we derive the privacy guarantees of \cref{algo:rw_gd} for arbitrary graphs and show how this leads to improved trade-offs for specific graphs widely used in decentralized learning. Our main result is a closed-form expression for the privacy loss between each pair of nodes, which holds for arbitrary graphs.

\begin{thm}
\label{thm:privacy}
    Consider a graph $G$ with transition matrix $W$ satisfying \cref{assum:privacy}. After $T$ iterations, for a level of noise $\sigma^2 \geq 2 \alpha (\alpha -1)$, the privacy loss of \cref{algo:rw_gd} from node $u$ to $v$ is bounded by:
        \begin{equation*}
        \label{eq:RW_pndp_loss}
            \eps_{u \rightarrow v} \leq \bigo \left( \frac{\alpha T \ln(T)}{\sigma^2 n^2} - \frac{\alpha T}{\sigma^2 n}\ln\Big(I-W + \frac{1}{n}\one\transp{\one}\Big)_{uv}\right)\,.
        \end{equation*}
\end{thm}

We recall that the logarithm of a matrix corresponds to the matrix whose eigenvalues are the log of the original eigenvalues and the eigenvectors remain identical. In particular, if $\lambda_1, \cdots, \lambda_n$ are $n$-eigenvalues (counting multiplicity) of a bistochastic symmetric matrix $M$ and $x_1, \cdots, x_n$ are the corresponding eigenvectors, then $\log(M) = \sum_{i=1}^n x_i x_i^\top\log(\lambda_i) $. 
Note that \cref{assum:privacy} ensures this matrix is well-defined. 

\begin{proof}[Sketch of proof]
    We give a high-level overview of the proof here and refer to \cref{app:privacy} for details. We fix the two vertices $u$ and $v$ and see how a token visit to $u$ will leak information to $v$. By the post-processing property of RDP, it is sufficient to compute the privacy loss that occurs when the token reaches for the first time $v$ after the visit in $u$. For computing this loss, we use the weak convexity of the Rényi divergence \cite{van2014renyi} to condition over the number of steps before reaching $v$. The length of the walk is parameterized by the power of the transition matrix. For a given length, we bound the privacy loss by using privacy amplification by iteration (\cref{thm:amp}). Refactoring the sum leads to the logarithm of the matrix. We finish by using composition over the $\bigo(T/n)$ times the token visits $u$ during the walk. 
\end{proof}

\begin{rmk}
    For the complete graph, the second term is equal to zero as the transition matrix is exactly $W = \frac{1}{n}\one\transp{\one}$. Thus, we recover the bound of~\citet{Cyffers2020PrivacyAB}. In other words, Theorem~\ref{thm:privacy} is a generalization of this previous result to arbitrary graphs.
\end{rmk}

\begin{rmk}
\label{rem:no_anon}
    \Cref{thm:privacy} holds for the definition of the view of the node given in \eqref{eq:view_main}, where the sender is kept anonymous. We provide in \cref{app:opensender} a similar theorem if the senders are known. In this case, although the formula is more complex, the asymptotic is the same as it roughly shifts the privacy guarantees from one hop.
\end{rmk}

\subsection{Interpretation of the Formula via Communicability}
\label{sec:formula_analysis}

The privacy loss of \cref{thm:privacy} has two terms that we can interpret as follows. The first term is the same as in~\citet{Cyffers2020PrivacyAB} for the complete graph: we have an $O(1/n^2)$ privacy amplification factor compared to local DP, matching what would be obtained in central DP with $n$ users. Our analysis reveals that this term also appears for arbitrary graphs: we can interpret it as a baseline privacy loss that occurs from the collaboration of all agents.

However, in graphs differing from the complete graph, this baseline privacy loss is corrected by the second term which depends on the specific pair $(u,v)$ of the nodes considered. Note that this second term can be negative for some pairs as eigenvector components can be of arbitrary signs. This quantity can be seen as a variant of known graph centrality metrics and, more precisely, communicability.

\begin{defn}[Communicability, \citealp{Estrada_2008, estradaknight}]
    \label{def:com}
    For a transition matrix $W$ and $c_i$ a non-increasing positive series ensuring convergence, the communicability between two vertices $u,v$ is defined by:
    \begin{equation*}
        \textstyle G_{uv} = \sum_{i=1}^{\infty} c_i W^i_{uv}\,.
    \end{equation*}
\end{defn}

For $c_i = 1/i!$, this corresponds to the original notion of communicability as presented in \citep{Estrada_2008}, while for $c_i = \alpha^i$ we recover the Katz centrality~\cite{Katz1953ANS}. Our formula corresponds to the case $c_i = \frac{\alpha}{\sigma^2 i}$ as proven in \cref{app:privacy}, and the convergence of the infinite sum is ensured by the fact that we remove the graph component associated with the eigenvalue $1$.

Communicability is used to detect local structures in complex networks, with applications to community detection and graph clustering, for instance, in physical applications \cite{estradaknight}. Good communicability is supposed to capture how well connected are the two nodes in the networks, i.e. how close they are.  Hence, having the second term of the privacy loss proportional to the communicability of the nodes shows that our formula matches the intuition that nodes leak more information to closer nodes than to more distant ones.

\subsection{Application to Specific Graphs}

We now give closed-form expressions of the privacy loss \eqref{eq:RW_pndp_loss} for specific graphs. 
To show the power of our bound, we note that the results for two network graphs studied in \citet{Cyffers2020PrivacyAB}, namely the complete graph and the ring graph, follow as a corollary of our general result. We illustrate below the possibility to derive closed formulas for other classes of graphs. %
Proofs are given in \cref{app:hypercube}.

\noindent \textbf{Star graph.} We consider a star graph with a central node in the first position linked to the $n-1$ other nodes. We choose the transition matrix such that the probability of self-loop $\selfloop>0$ is an arbitrarily small constant, and the distribution over the non-central nodes is uniform. Then we have the following privacy guarantees.

\begin{thm}
    Let $u,v \in V$ be two distinct nodes of the star graph and $\selfloop>0$ be an arbitrarily small constant. For a single contribution of node $u$ in \cref{algo:rw_gd} on the star graph, the privacy loss to node $v$ is bounded by:
    \begin{equation*}
\eps_{u \to v} \leq \begin{cases}
    -{\alpha(1-\kappa) \over \sigma^2 (n-1)}\ln \paren{ 1 - {1 \over n-1} } & u \neq 1 \text{ and } v \neq 1 \\
    {\alpha(1-\selfloop) \over 2\sigma^2\sqrt{n-1}} {\ln\paren{\sqrt{n-1}+1 \over \sqrt{n-1}-1}} & u=1 \text{ or } v=1 %
\end{cases}.
\end{equation*}
\label{thm:star}
\end{thm}
\vspace{-6mm}
\begin{proof}[Sketch of proof]
At a high level, as $\selfloop$ is an arbitrarily small constant, we can have an upper bound on the entries of all the powers of the adjacency matrix of the star graph.
\end{proof}

In particular, composing over the $\bigo(T/n)$ contributions, we see that extremal nodes enjoy a privacy amplification factor of order $\bigo(n^2)$ and the central node of order $\bigo(n)$. 

\noindent \textbf{Ring graph.} We consider a symmetric ring where nodes are enumerated from $1$ to $n$, which thus slightly differs from the case studied in \citet{Cyffers2020PrivacyAB, Yakimenka2022StragglerResilientDD}, where the ring is directed and thus deterministic \citep[up to the possibility of skipping in][]{Yakimenka2022StragglerResilientDD}. For this graph, our results shows that the amplification is parameterized by the distance between the nodes in the ring.
\begin{thm}
    \label{thm:ring}
    Let $u,v\in V$ be two distinct nodes of the ring with $a = (u+v-2) \mod n$ and $\alpha'=\alpha \mathbf 1_{|u-v|=1}$. For a single contribution of node $u$ in \cref{algo:rw_gd} on the ring graph, the privacy loss to node $v$, $\eps_{u \to v},$ is bounded by:
\begin{equation*}
\begin{split}
    {\alpha' \log(T) \cos\left({2\pi a\over n} \right)\over n \sigma^2}  +  {2\alpha \over n\sigma^2} \sum_{k=1}^{n-1} \cos \frac{\pi  ak}{n} \ln  \frac{3\csc^2(\pi k/n)}{4},
\end{split}
\end{equation*}
    in the case of equal probability between going left, right, or self-looping. Furthermore, if the probability of self-looping is set to $\selfloop>0$, then $\eps_{u \to v}$ is bounded by 
\begin{equation*}
\begin{split}
     {\alpha' \over n \sigma^2} 
     +  {\alpha(1-\selfloop) \over n \sigma^2} \sum_{t=2}^T \sum_{k=1}^n \cos^{t-1}{2\pi k \over n} \cos{2\pi (a+1)k \over n}.
\end{split}
\end{equation*}
\end{thm}
\begin{proof}[Sketch of proof]
At a high level, we use the fact that the adjacency matrix for a ring graph is a circulant matrix, so its eigenvectors are the Fourier modes. Therefore, its full spectral decomposition can be easily computed.
\end{proof}
To get an intuition regarding \Cref{thm:ring}, consider two nodes that are close to each other. For simplicity of calculation, consider nodes $1$ and $2$ and the statement of the theorem. Then $a=0$ and $\cos^{t-1}(2\pi k/n)\cos(2 \pi (a+1)k/n) = \cos^t(2\pi k/n)$. Therefore, $\eps_{\utov} \leq {\alpha \over \sigma^2n} + {\alpha(1-\selfloop) \over \sigma^2n}\sum_{k=1}^n {\cos^2(2\pi k/n) \over 1 - \cos(2\pi k/n)}$.

\vspace{-3mm}
\section{Experiments}
\label{sec:experiments}
In this section, we illustrate our results numerically on synthetic and real graphs and datasets and show that our random walk approach achieves superior privacy-utility trade-offs compared to gossip as long as the mixing time of the graph is good enough. The code is available at \url{https://github.com/totilas/DPrandomwalk}

\textbf{Privacy losses and comparison with the gossip counterpart.} We generate synthetic graphs with $n=2048$ nodes and report the privacy loss averaged over $5$ runs for every pair of nodes of the graphs as a function of the length of their shortest path in \cref{fig:synthe}. The transition matrix is computed using the {\em Hamilton weighting}. To compare with the private gossip algorithm~\cite{muffliato}, we consider the task of averaging (thus with $L=\mu=2$) and consider the same precision level and the same graphs: an exponential graph, an Erd\"os-Rényi graph with $q = c \log(n)/n$ for $c > 1$, a grid and a geometric random graph. Our private random walk approach incurs a smaller privacy loss for close enough nodes than the private gossip algorithm. Remarkably, our approach improves upon the baseline local DP loss even for very close nodes. Conversely, the privacy loss of our approach is generally higher for more distant nodes.
Nevertheless, random walks offer uniformly better privacy guarantees than gossip algorithms for graphs with good connectivity such as Erdös-Rényi graphs or expanders. 
\begin{figure}[t]
\centering
\caption{Comparison of privacy loss for random walks in bold lines and gossip in dashed lines for the same synthetic graphs with $n=2048$. Random walks allow privacy amplification even for very close nodes, but the decay is slower than for gossip.}
\includegraphics[width=0.9\linewidth]{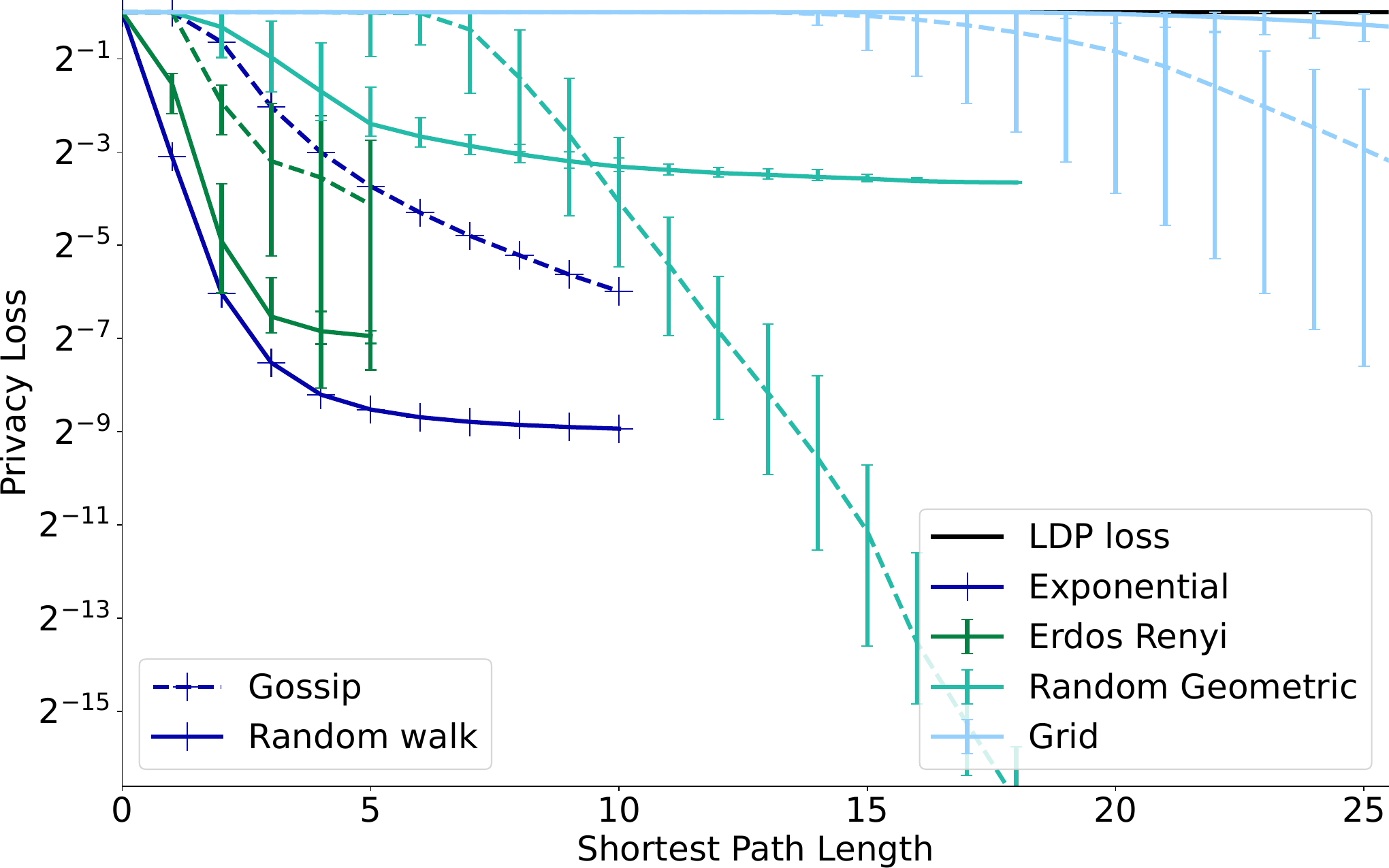}
\label{fig:synthe}
\vspace{-5mm}
\end{figure}

\textbf{Logistic regression on synthetic graphs.} We train a logistic regression model on a binarized version of the UCI Housing dataset.\footnote{\url{https://www.openml.org/d/823/}} The objective function corresponds to $f_v (x) = \frac{1}{|\mathcal{D}_v|}\sum_{(d, y)\in\mathcal{D}_v}\ln(1+ \exp(-y \transp{x} d))$ where $d \in \R^d$ and $y \in \{-1, 1 \}$. As in \citet{Cyffers2020PrivacyAB, muffliato} and \citet{stragglers}, we standardize the features,
normalize each data point, and split the dataset uniformly at random into a training set (80\%) and a test set (20\%). We further split the training set across $2048$ users, resulting in local datasets of $8$ samples each.

In a first experiment, we compare centralized DP-SGD, local DP-SGD, and our random walk-based DP-SGD. For all algorithms, we follow common practice and clip the updates to control the sensitivity tightly. We set $\eps =1$ and $\delta = 10^{-6}$. Following \citet{muffliato}, we use the mean privacy loss over all pairs of nodes (computed by applying \cref{thm:privacy}) to set the noise level needed for our random walk-based DP-SGD.
Figure~\ref{fig:houses} reports the objective function through iterations for complete, hypercube, and random geometric graphs. Experimentally, the behavior is the same for all the graphs, meaning that the token walk is diverse enough in every case to have a behavior similar to a uniformly random choice of nodes. The improvement in the privacy-utility trade-off compared to the local DP is significant.

\begin{figure}[t]
\centering
\caption{Private logistic regression on the Houses dataset where we compare our RW DP-SGD to with Local and Centralized DP-SGD as baselines.}
\includegraphics[width=0.8\linewidth]{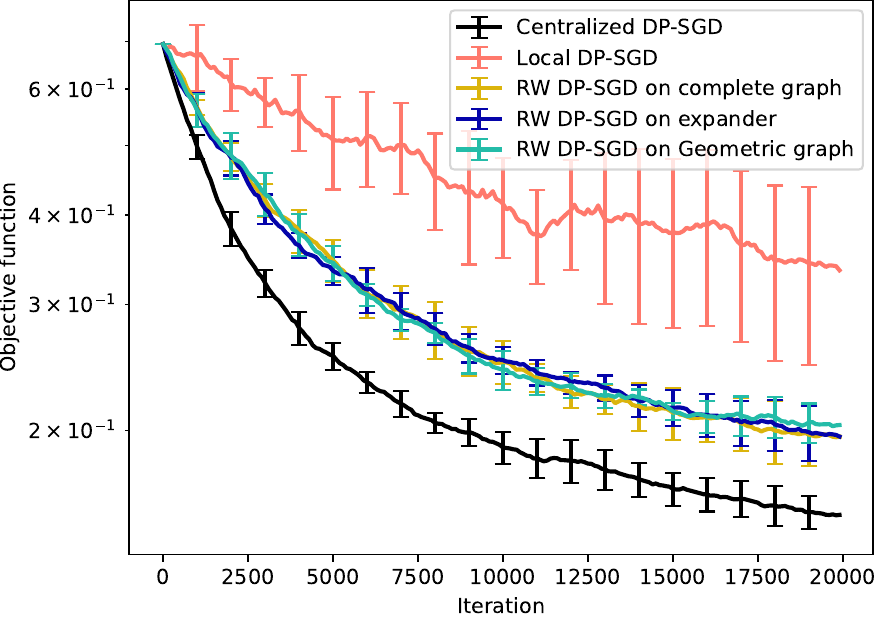}
\label{fig:houses}
\end{figure}

We then compare our random walk algorithm to its gossip counterpart \citep{muffliato} on the same logistic regression task. For both algorithms, we fix the mean privacy loss $\bar{\varepsilon}$ across all pairs of nodes to three different levels ($\bar{\varepsilon} \in \{0.5, 1, 2\}$) and we report the accuracy reached by each algorithm on $4$ graphs: complete, exponential, geometric and grid.
As shown in Table~\ref{tab:res}, our random walk algorithm outperforms gossip in all cases, which can be explained by a combination of two factors. First, as seen previously in Figure~\ref{fig:synthe}, our algorithm yields a lower mean privacy loss for graphs with good expansion property, especially when the degree of the nodes is high (because the privacy guarantees of gossip degrade linearly with the degree, while random walk is insensitive to it). This gain directly leads to less noise injection when fixing the mean privacy loss, and thus better utility. The second factor that explains why the gain in utility is so pronounced (even for the grid) comes from differences in the SGD version of gossip and random walk. In both methods, the privacy guarantee degrades as a function of the number of participations of nodes (linearly in Rényi DP). In gossip, allowing each node to participate $10$ times means that the model will essentially be learned by applying $10$ ``global'' gradient updates (i.e., aggregated over the $n$ nodes), because all local gradients are gossiped until convergence between each gradient computation \citep[see Algorithm 3 in][]{muffliato}. In our random walk algorithm, the model is learned by applying $10 \times n$ ``local'' gradient updates. Even if this represents the same amout of information about the data, better progress is made with many noisy steps than with a small number of less noisy steps, in the same way as mini-batch SGD tends to progress faster than GD in practice.

\begin{table}[t]
    \centering
    \caption{Model accuracy at various mean privacy loss levels averaged over $8$ runs.}
    \begin{tabular}{lcc}
        \toprule
        \textbf{Graph} & \textbf{Gossip} & \textbf{Random walk} \\
        \midrule
        \multicolumn{3}{c}{\textbf{Mean Privacy Loss of 0.5}} \\
        Complete & $0.65\pm 0.10$ & $0.841\pm 0.07$ \\
        Exponential & $0.70 \pm 0.10$ & $0.818\pm 0.09$ \\
        Geometric & $0.60 \pm 0.07$ & $0.795 \pm 0.06$ \\
        Grid & $0.60 \pm 0.07$ & $0.803 \pm 0.10$ \\
        \midrule
        \multicolumn{3}{c}{\textbf{Mean Privacy Loss of 1}} \\
        Complete & $0.70\pm 0.10$ & $0.900\pm 0.04$ \\
        Exponential & $0.77 \pm 0.05$ & $0.883 \pm 0.05$ \\
        Geometric & $0.66 \pm 0.10$ & $0.873 \pm 0.05$ \\
        Grid & $0.73 \pm 0.10$ & $0.848 \pm 0.07$ \\
        \midrule
        \multicolumn{3}{c}{\textbf{Mean Privacy Loss of 2}} \\
        Complete & $0.83\pm 0.06$ & $0.940\pm 0.02$ \\
        Exponential & $0.89 \pm 0.04$ & $0.937 \pm 0.01$ \\
        Geometric & $0.67\pm 0.04$ & $0.933 \pm 0.02$ \\
        Grid & $0.72 \pm 0.10$ & $0.919 \pm 0.02$ \\
        \bottomrule
    \end{tabular}
    \label{tab:res}
\end{table}

\begin{figure*}[t]
\centering \vspace{-9pt}
\caption{Link between graph structure and privacy loss. Left: (a) example of Facebook Ego graph communicability and privacy loss, logarithmic scale. Middle: (b) same on the Southern women graph. Right: (c) the corresponding mean privacy loss and Katz centrality.}

\hspace{-5em}
\subfigure[]{
    \includegraphics[width=0.295\linewidth]{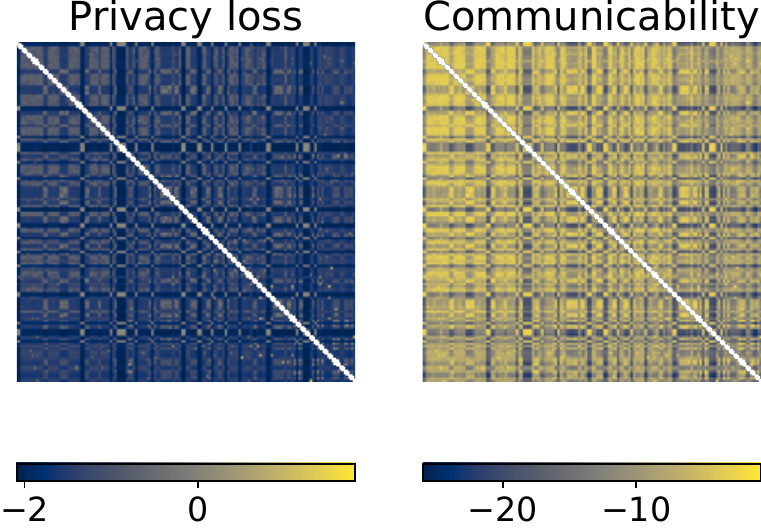}
    \label{fig:fb}
}
\hspace{-.8em}
\subfigure[]{
    \includegraphics[width=0.29\linewidth]{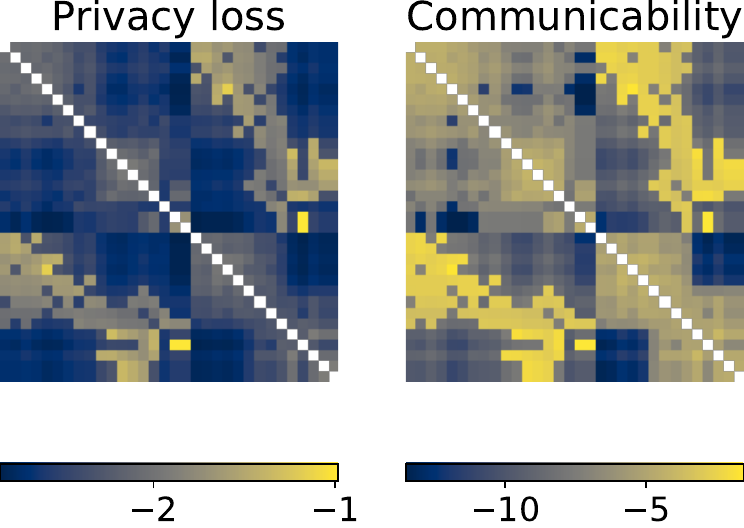}
    \label{fig:south}
}
\hspace{-.5em}
\subfigure[]{
    \includegraphics[width=0.31\linewidth]{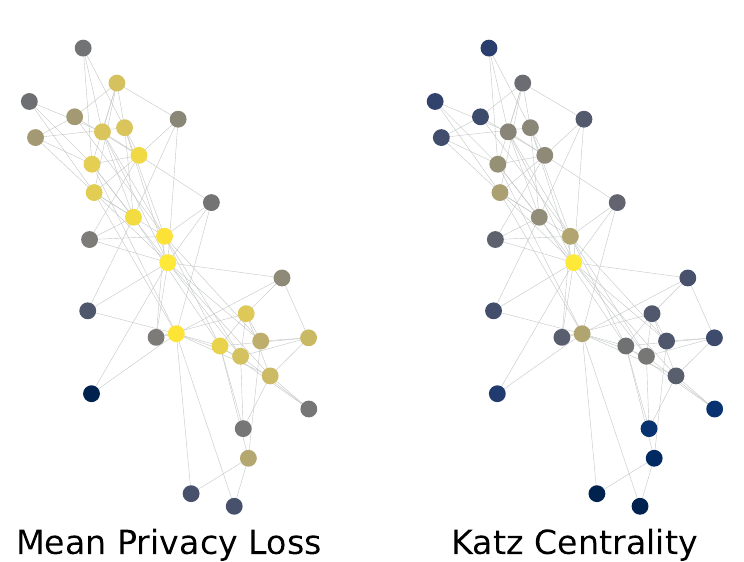}
    \label{fig:southgraph}
}
\vspace{-9pt}
\end{figure*}

\textbf{Privacy loss on real graphs and communicability.} 
We consider two real-world graph datasets well-suited for community detection: (i) The Facebook Ego dataset~\cite{egofbgraph} represents subgraphs of the Facebook network, where users are nodes and edges correspond to the friendship relation, and each subgraph corresponds to the set of friends of a unique user that is removed from the subgraph; (ii) The Davis Southern women social network~\cite{womendata} is a graph with $32$ nodes which corresponds to a bipartite graph of social event attendance by women and has been used in \citet{koloskova2019decentralized} and \citet{pasquini2023insecurity}. We report side-by-side the matrix of pairwise privacy losses and of the communicability in \cref{fig:fb} and \cref{fig:south}. This confirms the similarity between the two quantities as discussed in Section~\ref{sec:formula_analysis}. We also report the Katz centrality compared to the mean privacy loss for each node in \cref{fig:southgraph}, showing that the two quantities also have similar behavior.

We provide other numerical experiments in \cref{app:expes}: we report the privacy loss on the other Facebook Ego graphs and study the impact of data heterogeneity.

\section{Conclusion}

In this work, we analyzed the convergence and privacy guarantees of private random walks on arbitrary graphs, extending the known results for the complete graph and the ring. Our results show that random walk-based decentralized algorithms provide favorable privacy guarantees compared to gossip algorithms, and establish a link between the privacy loss between two nodes and the notion of communicability in graph analysis. Remarkably, as long as the spectral gap of the communication graph is large enough, the random walk approach nearly bridges the gap between the local and the central models of differential privacy. Our results could be broadened by showing convergence under more general hypotheses. Other extensions could include skipping some nodes in the walk, considering latency, or having several tokens running in parallel.

\section*{Impact Statement}
This work contributes to proving formal privacy guarantees in machine learning and thus paves the way for the larger adoption of privacy-preserving methods, with a better theoretical understanding. In particular, our work shows that decentralized algorithms may be useful in designing efficient privacy-preserving systems with good utility. In particular, the impact of the choice of a decentralized algorithm on privacy guarantees has not been studied so far. We believe our work opens interesting directions to see how one could design decentralized protocols that are private by design.

A potential risk of misuse of our work is that we rely on different privacy budgets across different pairs of users, as done previously in~\citet{muffliato}. This could sometimes lead to a false sense of security or weaker privacy guarantees than those provided in other threat models. In particular, the question of comparing the privacy guarantees we provide to those of the central model could be discussed. Our experiments use the mean privacy loss of nodes to set the comparison, but it is interesting to note that for the same mean privacy loss, the distribution of the privacy loss for gossip and random walk algorithms is not the same. Indeed, as shown in Figure~\ref{fig:synthe}, random walks have a smoother decay but can enjoy privacy amplification even for closely related nodes, whereas gossip tends to saturate the closest nodes then, followed by exponential decay. It is an interesting question to know how to interpret these privacy loss functions to choose the algorithm whose privacy guarantees best meet the privacy expectations in a particular use case.

\section*{Acknowledgments}

This work was supported by grant ANR-20-CE23-0015 (Project PRIDE), the
ANR-20-THIA-0014 program ``AI\_PhD@Lille'' and the ANR 22-PECY-0002 IPOP 
(Interdisciplinary Project on Privacy) project of the Cybersecurity PEPR. JU research was supported by a Decanal Research Grant from Rutgers University.
\bibliography{biblio}
\bibliographystyle{icml2024}

\newpage
\appendix
\onecolumn

\setcounter{page}{1}

\appendix
\section{Proof of \cref{thm:optim}}
\label{app:optim}

We adapt the proof of Theorem 5 of \cite{even2023stochastic}. Because constants matters in differential privacy, we detail the steps needed to explicit these constants. We recall here the keys steps and definition and how we adapt the proofs to allow the addition of Gaussian noise. We keep the same definitions, except that the sequence of token iterates is now defined by
\begin{equation*}
    x_{t+1} = x_t - \gamma (g_t + \eta_t)\,.
\end{equation*}
This does not impact Lemma 8, which only relies on the graph properties and the function at the optimum. Up to the transposition of notation, we have for any $T \geqslant 1$:
\begin{equation}
\mathbb{E}\left[\left\|\sum_{t<T} \nabla f_{v_t}\left(x^{\star}\right)\right\|^2\right] \leqslant T \zeta_{\star}^2+\zeta_{\star}^2 \sum_{t<T} \mathrm{~d}_{\mathrm{TV}}\left(P_{v_0,}^t, \pi^{\star}\right)+2 \zeta_{\star}^2 \sum_{s<t<T} \mathrm{~d}_{\mathrm{TV}}(t-s),
\end{equation}
where $\mathrm{d}_{\mathrm{TV}}(r)=\sup \left\{\mathrm{d}_{\mathrm{TV}}\left(\left(P^r\right)_{v,,}, \pi^{\star}\right), v \in \mathcal{V}\right\}$ for $r \in \mathbb{N}$, so that:
\begin{equation}
\mathbb{E}\left[\left\|\sum_{t<T} \nabla f_{v_t}\left(x^{\star}\right)\right\|^2\right] \leqslant \zeta_{\star}^2\left(4 \tau_{\text {mix }}(1 / 4)+T\left(1+8 \tau_{\text {mix }}(1 / 4)\right)\right) .
\end{equation}

Next, we can transform Lemma 9 into
\begin{equation*}
    \E_{\eta} (\norm{x_{t+1} - y_{t+1}}^2) \leq (1 - \gamma \mu) \E_{\eta} (\norm{x_t - y_t}^2) + \gamma L \norm{y_t - x^*}^2 + \gamma^2 (d\sigma^2 + \sigma_{sgd}^2)\,.
\end{equation*}
where the sequence of $y_t$ and $y_{t+1}$ satisfies the relation.
\begin{equation*}
    y_{t+1} = y_t - \gamma \nabla f_{v_t}(x^*)
\end{equation*}

By applying the formula recursively, we obtain:
\begin{equation*}
     \E_{\eta} \left\|x_T-y_T\right\|^2 \leqslant(1-\gamma \mu)^T\left\|x_0-y_0\right\|^2+ \sum_{t<T}(1-\gamma \mu)^{T-t} \left(\gamma L \left\|y_t-x^{\star}\right\|^2 + \gamma^2 (d\sigma^2+\sigma_{sgd}^2) \right)\,,
\end{equation*}
where we use $\E_{\eta}$ to denote the expected value with respect to the privacy noise.

By instantiating the $y$ sequence as done in the non-private version with the $x$, we recover nearly the same formula, with $y_t=x^*$.
The first term can be handled as in the non-private case, while the second term has an additional sum:
\begin{equation*}
    \begin{aligned}
\mathbb{E}_{\eta}\left\|x_T-x^{\star}\right\|^2 \leqslant & 2(1-\gamma \mu)^T\left(\mathbb{E}\left\|x_0-x^{\star}\right\|^2+\gamma^2 \mathbb{E}\left[\left\|\sum_{s<T} \nabla f_{v_s}\left(x^{\star}\right)\right\|^2\right]\right) \\
& +\sum_{t<T}(1-\gamma \mu)^{T-t} \left( \gamma^3 L\mathbb{E}\left[\left\|\sum_{t \leqslant s<T} \nabla f_{v_s}\left(x^{\star}\right)\right\|^2\right] + \gamma^2 (d\sigma^2+\sigma_{sgd}^2)\right).
\end{aligned}
\end{equation*}

As $\sum_{t<T}(1-\gamma \mu)^{T-t} < \frac{1}{\gamma \mu}$, and the other terms remain identical, we have:
\begin{equation*}
    \E_{\eta} \left\|x_T-y_T\right\|^2 \leqslant 2(1-\gamma \mu)^T\left\|x_0-x^{\star}\right\|^2+\frac{3 \gamma L}{\mu^2} C \tau_{\operatorname{mix}} \zeta_{\star}^2 + \frac{\gamma (d\sigma^2+\sigma_{sgd}^2)
    }{\mu}\,.
\end{equation*}
with $C=13$.

We conclude by plugging back the following $\gamma$ (as in \cite{even2023stochastic}) in the previous formula:
\begin{equation*}
    \gamma=\min \left(\frac{1}{L}, \frac{1}{T \mu} \ln \left(T \frac{\left\|x_0-x^{\star}\right\|^2}{\frac{39}{\mu^2} C \tau_{\operatorname{mix}} \zeta_{\star}^2}\right)\right)\,.
\end{equation*}

\begin{rmk}
As long as $\sigma^2 \leq \frac{39L\tmix \zeta_*^2}{d \mu}$, the noise due to privacy is smaller than the one due to the randomness of the walk.
\end{rmk}

\section{Privacy Proofs}
\label{app:privacy}
Let $u,v$ be two distinct nodes.
To prove \cref{thm:privacy}, we see the privacy loss $\eps_{u \to v}$ as the composition of the privacy loss induced by each of the contributions of node $u$. Thus, we first bound the privacy loss for one contribution $\eps^{\mathsf{single}}_{u \to v}$ . Let us denote by $t_c$ the time step where this contribution is made (i.e., the token is at node $u$ at time $t_c$). For $t\leq t_c$, there is no privacy leakage. 
Let us denote by $t_l$ the first $t \geq t_c$ where the token is held by $v$. By the post-processing property of differential privacy, the steps after $t_l$ will not yield additional leakage for the contribution of time $t_c$. Hence, we only need to bound the privacy leakage at time $t_l$. This leakage depends on the number of steps between $t_c$ and $t_l$. We use the weak convexity property of the Rényi divergence to decompose our privacy loss.

\begin{lemma}[Weak convexity of Rényi divergence]
    Let $\mu_{1}, \ldots, \mu_{n}$ and $\nu_{1}, \ldots, \nu_{n}$ be probability
    distributions over some domain $\mathcal{Z}$ such that for all $i \in[n], 
    \mathrm{D}_{\alpha}\left(\mu_{i} \| \nu_{i}\right) \leq c /(\alpha-1)$ for
    some $c \in(0,1]$. Let $\rho$ be a probability distribution over $[n]$ and denote by $\mu_{\rho}$ (respectively, $\nu_{\rho}$) the probability distribution over $\mathcal{Z}$ obtained by sampling $i$ from $\rho$ and then outputting a random sample from $\mu_{i}$ (respectively, $\nu_{i}$). Then
    \begin{equation}
    \mathrm{D}_{\alpha}\left(\mu_{\rho} \| \nu_{\rho}\right) \leq(1+c) \cdot \underset{i \sim \rho}{\mathbb{E}}\left[\mathrm{D}_{\alpha}\left(\mu_{i} \| \nu_{i}\right)\right]\,.
    \end{equation}
\end{lemma}

We can thus partition according to the length of the walk and write the privacy guarantee depending on  $T$ the total number of steps and $\beta(i)$ a function bounding the privacy loss occurring for seeing the token $i$ steps after the node contribution:
\begin{equation}
    \eps^{\mathsf{single}}_{u \rightarrow v} \leq (1+c) \sum_{i=1}^T \Pro(u \rightarrow v \text{ after } i \text{
steps}) \beta(i)\,.
\end{equation}

The probability of the path of $t$ steps between $u$ and $v$ can be easily extracted from the power of the transition matrix $W$. We thus obtain the generic formula.

\begin{lemma}
    Let $\beta(i)$ be a bound on the privacy loss occurring for seeing the token $i$ steps, and $T$ the total number of steps. Then, the following holds:
    \begin{equation}
    \label{eq:e_uv_lemma}
        \eps^{\mathsf{single}}_{u \to v} \leq   \sum_{i=1}^T W_{uv}^i 2\beta(i)\,.
    \end{equation}
\end{lemma}

We recognize in this formula the communicability (\cref{def:com}) with $c_i = 2\beta(i)$.

We can now compute the function $\beta$ by resorting to privacy amplification by iteration (\cref{thm:amp}). As we apply the Gaussian mechanism at each step, we have $s_1 = \frac{\alpha}{2\sigma^2}$ and $s_j = 0$ for $0 < j \leq i$. We thus take all $a_j = \frac{\alpha}{2\sigma^2 i}$ in \cref{thm:amp}. This gives the bound $\beta(i) = \frac{\alpha}{2\sigma^2 i}$, which corresponds to setting $c_i = {\alpha}{\sigma^2 i}$.

We now focus on how to simplify and compute this formula under \cref{assum:privacy}, i.e., the transition matrix is bistochastic and symmetric.
Since the matrix is symmetric by assumption, we can apply the spectral theorem to write:
\begin{equation}
    W = \sum_{i=1}^n \lambda_i \phi_i \transp{\phi_i}\,,
\end{equation}
Furthermore, since the matrix is bistochastic, $\lambda_1 = 1 > \lambda_2 \geq \dots \leq \lambda_n > -1$
and the eigenvector associated to the first eigenvalue is $\frac{1}{\sqrt{n}} \one$. Hence, we can isolate the first term and plug into \eqref{eq:e_uv_lemma} to get:
\begin{equation}
    \eps^{\mathsf{single}}_{u \to v} \leq \sum_{t=1}^{T} \frac{\alpha}{\sigma^2 t}\frac{1}{n} + \sum_{i=2}^{n} \sum_{t=1}^{T} \frac{\alpha}{\sigma^2 t} \lambda_i^t \phi_i (u)\phi_i(v)\,.
\end{equation}

In these sums, we isolate $\sum_{t=1}^{T} \lambda_i^t/t$. Noticing that the sum converges for all these eigenvalues, we can rewrite it as $\sum_{t=1}^{T} \lambda_i^t/t = - \ln(1 - \lambda_i) + \bigo(\lambda_i^{T})$. We use the integral test for convergence to replace the sum by the logarithm.

This gives us the privacy loss for a single contribution. 
In order to conclude, we compose the privacy loss of each of the $N_u$ contributions of node $u$, leading to, for any $\sigma^2 \geq 2 \alpha (\alpha -1)$ :
\begin{equation}
    \eps_{u \to v}\leq \frac{\alpha N_u \ln(T)}{\sigma^2 n} - \frac{\alpha N_u}{\sigma^2} \ln(I - W + \frac{1}{n} \one \transp{\one})_{uv} + \bigo(\lambda_2^T)\,.
\end{equation}
As $\abs{\lambda_2}<1$, its power decreases exponentially fast with the number of steps and thus is negligible with respect to the other terms. 

The average number of contributions is $T/n$ as the transition matrix is bi-stochastic. We can then upper bound the real number of contributions with high probability, for example by using Theorem 12.21 of \cite{levin2017markov}. The small probability of exceeding the upper bound can be added to $\delta$ when converting from RDP to $(\eps, \delta)$-differential privacy.

\begin{rmk}
    The upper bound on $N_u$ tends to be large for ``cryptographically'' small $\delta$. An efficient way to avoid this issue in practical implementations is to force a tighter bound on the maximum number of contributions by each node.
    After a node reaches its maximum number of contributions, if the token passes by the node again, the node only adds noise. We use this trick in numerical experiments.
\end{rmk}

\subsection{Adaptation to the case without sender anonymity}
\label{app:opensender}

For bounding the privacy loss of a single contribution, we consider above the value of the privacy loss occurring at time $t_l$.
However, this is computed without taking into account the knowledge of where the token comes from, for example, if $v$ can know which of its neighbors sent him the token. This computation is thus justified only in the specific case where the anononymity of the sender is ensured, e.g., by resorting to mix networks~\citep{mixnet} or anonymous routing~\citep{tor}.

If this is not the case, then we should \emph{a priori} use the conditional probability towards the position of the token at time $t_l - 1$. However, there is no close form for this in general for a graph. To fix this issue, we consider that the last step for reaching $v$ is only a post-processing of the walk reaching one of its neighbors.

For each of the neighbors, the previous formula applies. We obtain different values for the various neighbors, so a worst-case analysis consists in taking the max over this set. Denoting by $\widetilde{\varepsilon_{u v}^{\mathsf{single}}}$ the privacy loss when a node $v$ knows the neighbors from which it received the token, we have
$$
\widetilde{\varepsilon_{u v}^{\mathsf{single}}} \leqslant \max _{w \in \mathcal{N}_{v}} \varepsilon_{u w}^{\mathsf{single}}
$$

This approach allows to keep a closed form for the matrix and just add a max step. However, this analysis is not tight and may lead to a significant cost in some scenarios. A simple example is the special case where the nodes $u$ and $v$ are neighbors. It effectively assumes that the transition between $u$ and $v$ is always direct, which may not always be the case.

We provide in \cref{fig:synthemax} the equivalent of \cref{fig:synthe} by taking the max below. As expected, amplification is smaller for close nodes, but the curves have the same asymptotic.

\begin{figure}[ht]
\centering
\includegraphics[width=0.4\linewidth]{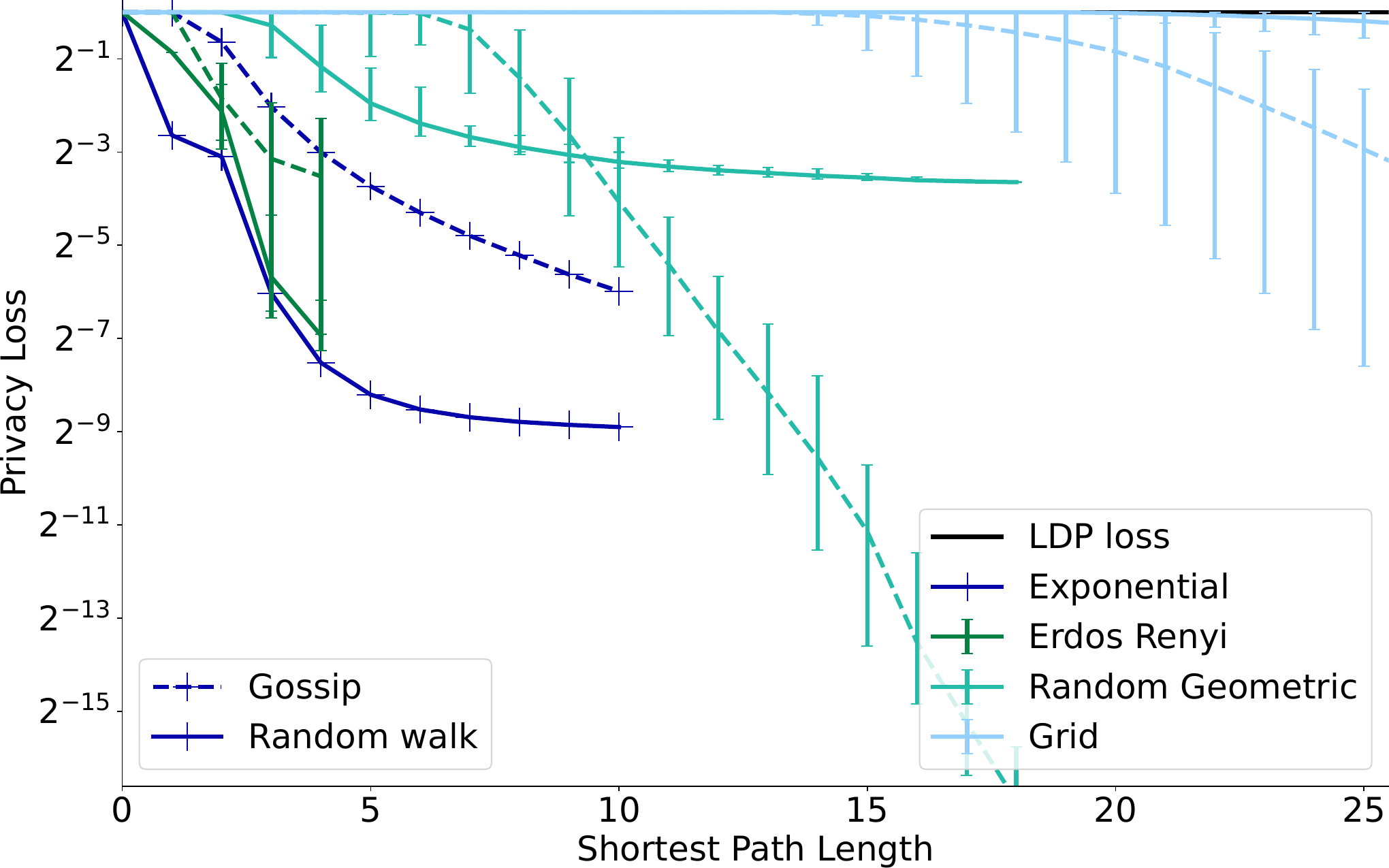}
\caption{Comparison of privacy loss for random walks when nodes know who send them the token in bold lines and gossip in dashed lines for the same synthetic graphs with $n=2048$. Privacy amplification is visible even for close neighbors but the decay is slower than for gossip.}
\label{fig:synthemax}
\end{figure}

\section{Additional Numerical Experiments}
\label{app:expes}

In this section, we include other examples of graphs that illustrate how the privacy loss matches the graph structure. 

A classic synthetic graph to exhibit subgroup is the stochastic block model, where each edge follows an independent Bernoulli random variable. The parameter of the law depends from a matrix encoding the relation between the clusters. As for other graphs, the privacy loss is similar to communicability and shows different level of privacy within and outside a group \cref{fig:sbm}, so that a node sees the major part of its privacy loss occurring within its group (\cref{fig:sbmgraph}.

\begin{figure}
\centering
\hspace{-1.5em}

\subfigure{
    \includegraphics[width=0.4\linewidth]{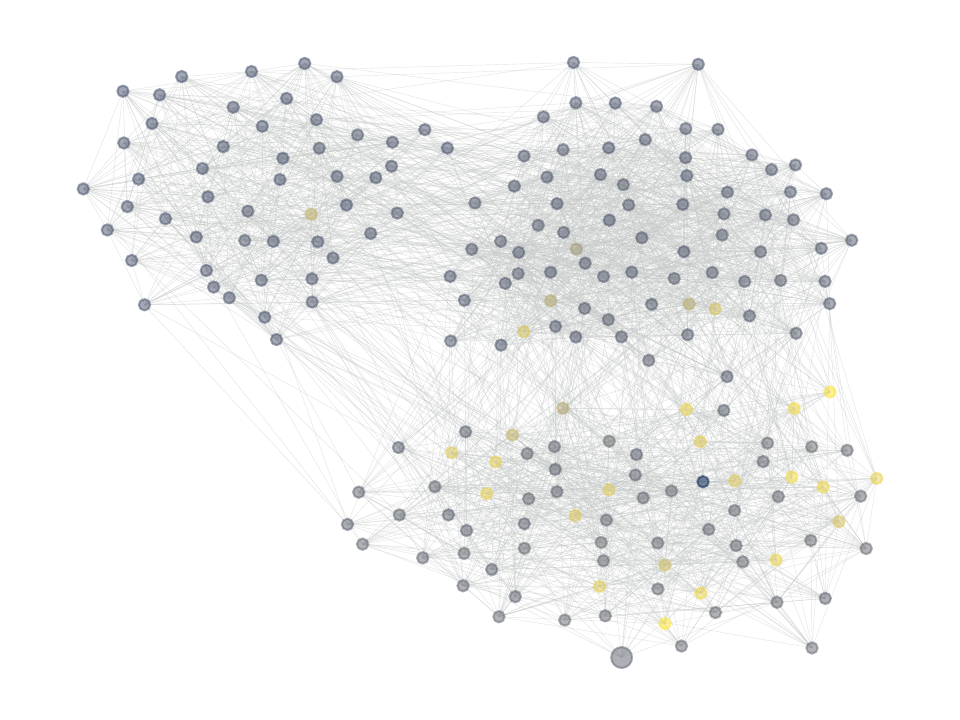}
    \label{fig:sbmgraph}
}
\hspace{-.8em}
\subfigure{
    \includegraphics[width=0.4\linewidth]{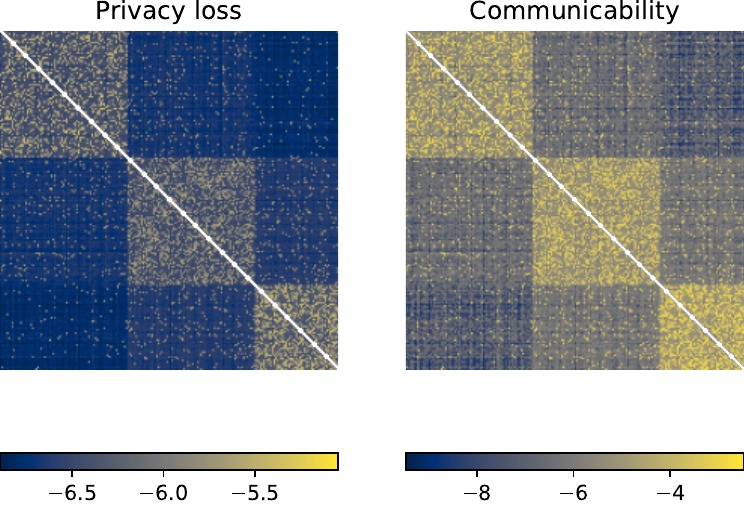}
    \label{fig:sbm}
}

\caption{Stochastic Block Model with $200$ nodes in three clusters $(75,75, 50)$ and probability matrix $[[0.25, 0.05, 0.02], [0.05, 0.35, 0.07], [0.02, 0.07, 0.40]]$. The privacy loss matrix recovers the different blocks. 
}
\end{figure}

We report other examples of privacy loss for a node chosen at random in Facebook Ego graphs. Once again, these results illustrate that our privacy loss guarantees match the existing clusters of the graph (\cref{fig:allego}).

\begin{figure}
    \centering
    \includegraphics[width=\textwidth]{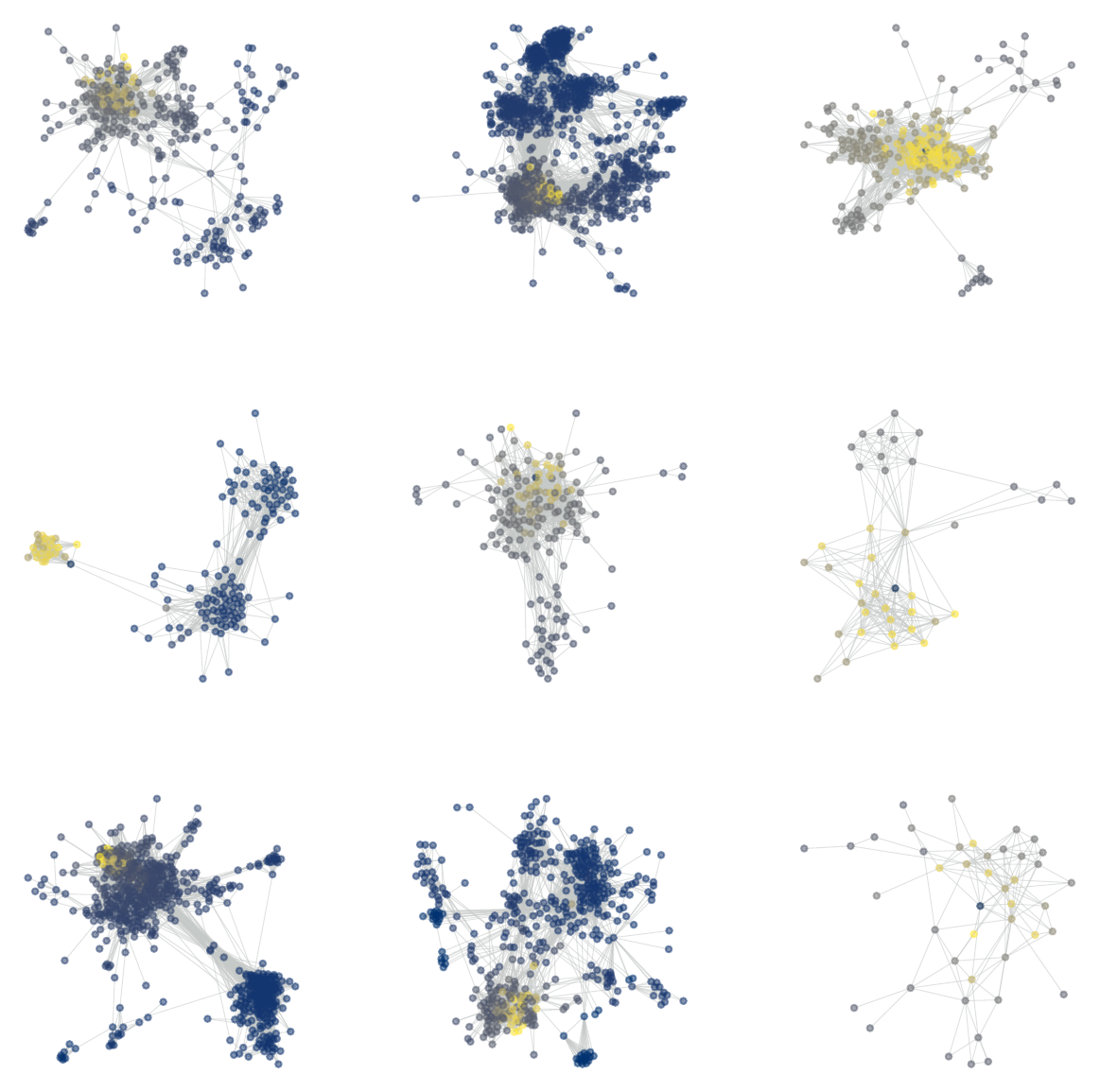}
    \caption{Privacy loss on the 9 other Facebook Ego graphs, following the same methodology as in Figure~\ref{fig:fb}.}
    \label{fig:allego}
\end{figure}

Intuitively, compare to gossip algorithms where the updates only slowly flows in the graph, the random walk should mix quite easily heterogeneous data. while we do not derive mathematical guarantees on heterogeneity, we illustrate this idea with the following numerical experiment. We generate a synthetic geometric random graph and compare two scenarii (\cref{fig:graphrandom}). On the first trial, the data is position-dependent, which generate heterogeneity as close nodes also have close datapoint. We then shuffle randomly the data to destroy the heterogeneity and compare the convergence of the two in \cref{fig:ageozgeo}. 

\begin{figure}
    \centering
    \includegraphics[width=.8\textwidth]{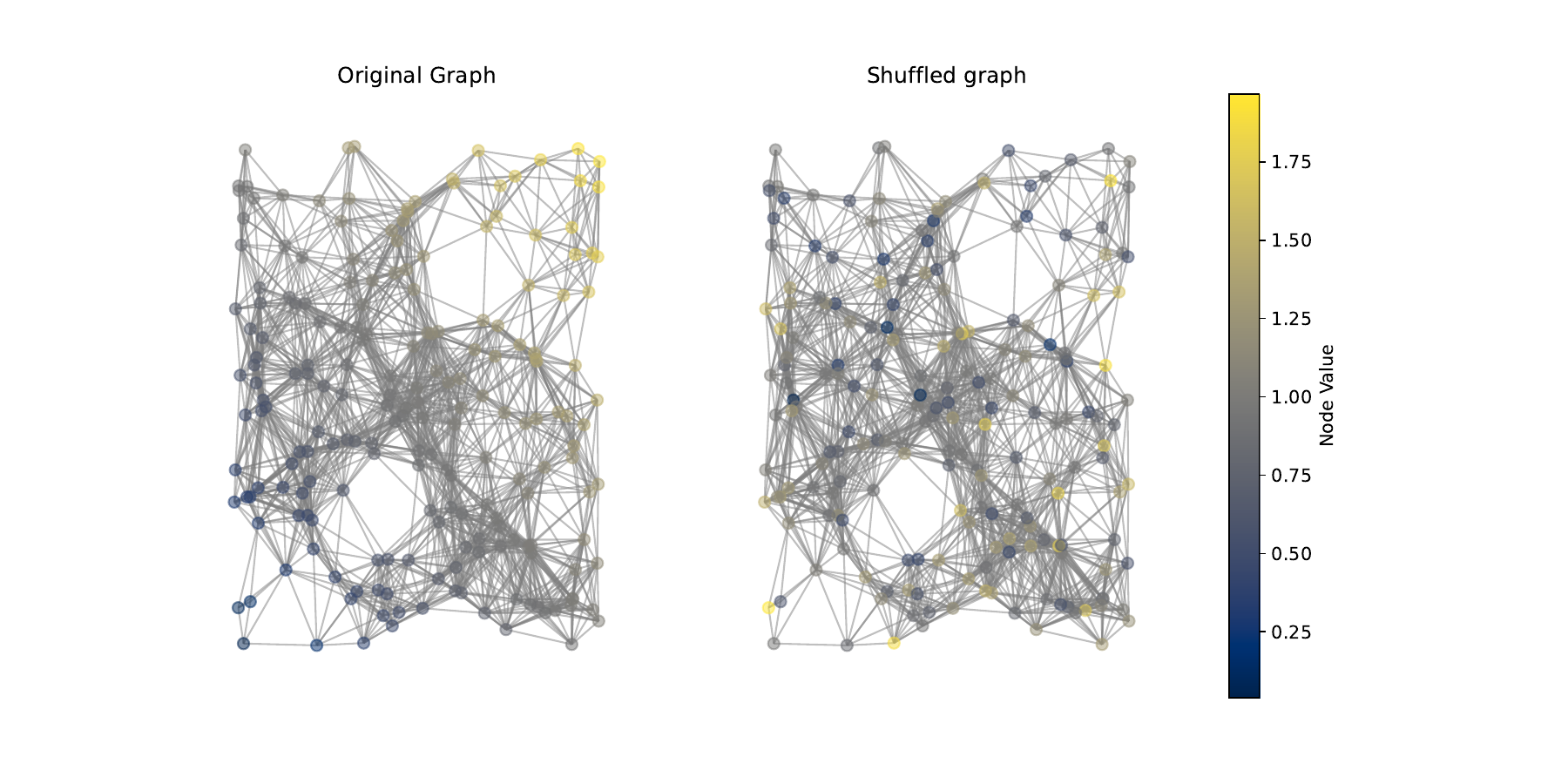}
    \caption{Geometric random graph with $200$ nodes. On the left, the label is given by the sum of the coordinates, providing heterogeneity in the graph. On the right the same graph has its label shuffled}
    \label{fig:graphrandom}
\end{figure}

\begin{figure}
    \centering
    \subfigure[With original labels]{
        \includegraphics[width=.4\textwidth]{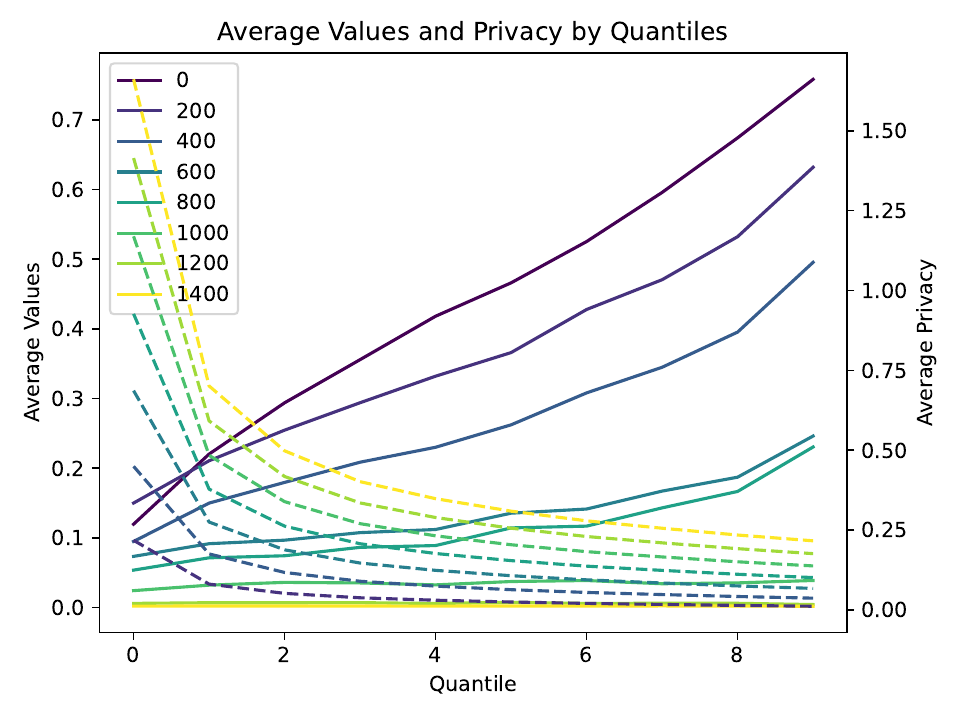}
    }
    \hfill
    \subfigure[With shuffled labels]{
        \includegraphics[width=.4\textwidth]{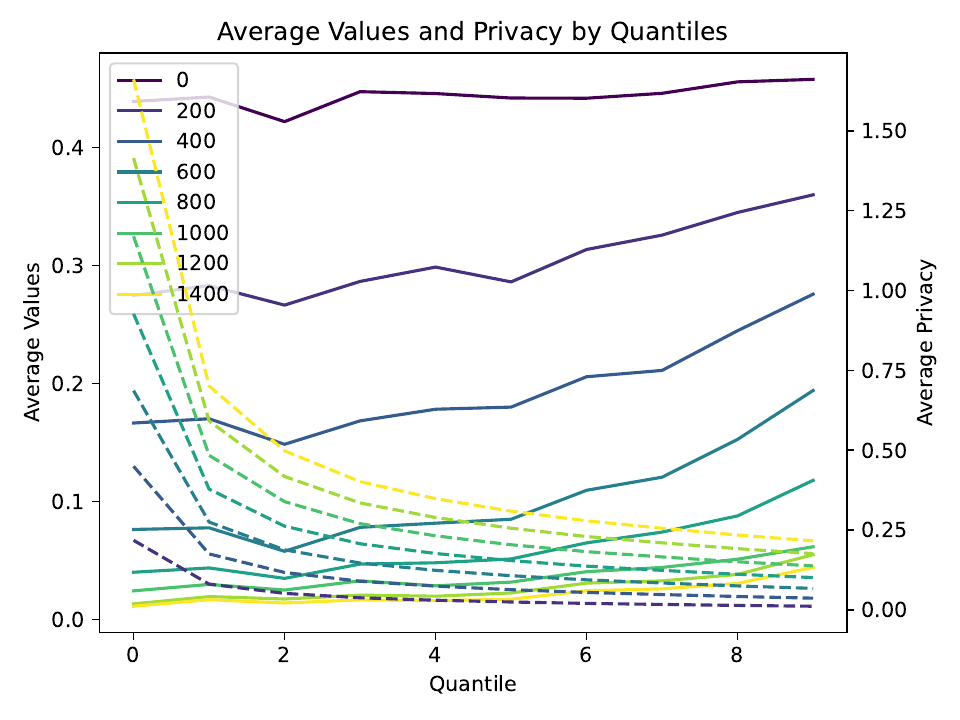}

    }
    \vskip -.5 em
    \caption{We compute the quantiles of the euclidean distance between node. For each quantile, we report the mean on all the pair of nodes of the quantiles for the privacy loss and for the distance between the current estimates. We report different time step across the learning. The privacy loss is identical in both cases. At the beginning of the learning, the homogeneous case has smaller average values heterogeneity, but the difference reduces with the learning}
    \label{fig:ageozgeo}
\end{figure}

\section{Collusion}
\label{app:collusion}

The results of our work assume that nodes are separated entities that do not share information outside the protocols. One could however claim that a fraction of nodes can collude and share information between them. In this case, if we denote $F \subset V$ the fraction of the colluded nodes, for a given contribution done by $u$ what matters is the first time that one of the node of $F$ is reached afterwards. More precisely, we can derive the privacy loss as
\begin{equation}
    \eps^{single}_{u \to F} \leq  \sum_{i=1}^T \left(\sum_{v \in F}  W_{uv}^i\right) {\alpha}{\sigma^2 i}\,.
\end{equation}
where the term between parenthesis corresponds to the probability to reach $F$ in exactly $i$ steps from $u$. By reorganizing these terms, we obtain the upper bound:
\begin{equation}
    \eps_{u \to F} \leq \sum_{v \in F} \eps_{u \to v}
\end{equation}
This term corresponds also to the formula one would obtain from basic composition. Hence, our analysis does not allow to avoid the degradation of the privacy guarantees to collusion. Note that if all the colluded nodes are far away from $u$, it is still possible to derive non trivial guarantees compare to what would give the bound of local differential privacy. In comparison to gossip where the privacy loss decrease is sharper with distance, the cases where the amplification remains are scarcer.  This is a fundamental limitation of amplification by decentralization, that was already pointed out in \citep{Cyffers2020PrivacyAB, muffliato}.

\section{Refined Privacy Bounds for Specific Graphs}
\label{app:hypercube}

\subsection{Useful Auxiliary Results}
We first collect some auxiliary results that we use in our bounds. 

\begin{prop}
\label{prop:oddevensum}
    For any $x \in (0,1)$, we have 
    \[
    \sum_{p \text{ is odd}} {x^p \over p} = \frac{1}{2}\ln \frac{1+x}{1-x} \quad \text{and} \quad \sum_{p \text{ is even}} {x^p \over p} = - \ln(1- x^2)\,.
    \]
\end{prop}
\begin{proof}
Recall that 
\[ \ln(1+x) = x + {x^2 \over 2} + {x^3 \over 3} + \cdots, \quad \text{and} \quad \ln(1-x) = -x + {x^2 \over 2} - {x^3 \over 3} + \cdots \]
Now $\ln(1+x) - \ln(1-x)$ gives the first bound and $\ln(1-x) + \ln(1+x)$ gives the second bound. This completes the proof. 
\end{proof}

\begin{prop}
[Godsil and Royle~\cite{godsil2001algebraic}]
\label{prop:godsilroyle}
    Let $1 \leq d \leq n-1$ be an integer. Then for any $d$-regular graph $G$, the eigenvectors of the Laplacian and those of the adjacency matrix of $G$ coincide. 
\end{prop}

\subsection{Privacy Loss for Specific Graphs}
\label{app:specific_graphs}

\paragraph{Complete graph.} The transition matrix is exactly $\frac{1}{n} \one \transp{\one}$ and thus we recover exactly the same formula as in~\cite{Cyffers2020PrivacyAB}. In particular, there is only one non-zero eigenvalue with magnitude $1$ with an all-one vector as the corresponding eigenvector. In particular,
\begin{align*}
    \eps_{u \to v} &\leq {\alpha \over \sigma^2} \sum_{i=1}^T W_{uv}^i {1 \over i} = {\alpha \over \sigma^2} \sum_{i=1}^T  {1 \over i} \paren{\sum_{j=1}^n \lambda_j^i v_jv_j^\top }_{uv} = {\alpha \over n\sigma^2} \sum_{i=1}^T  {1 \over i} \leq  {\alpha \log(T) \over n\sigma^2}
\end{align*}

\paragraph{Ring graph.} To ensure aperiodic Markov chain as required by \cref{assum:optim}, the transition matrix should be in the form $aI + b(J + \transp{J}), a+2b=1$.

The adjacency matrix $A_R$ of the ring graph $R$ is a circulant matrix. Therefore, all its eigenvectors are just the Fourier modes~\cite{horn2012matrix}:
\[
\phi(\omega_k) = 
\begin{pmatrix}
    1 \\
    \omega_k \\
    \omega_k^2 \\
    \vdots \\
    \omega_k^{n-1}
\end{pmatrix},
\]
where $\omega_k^n=1$ is the $n$-th root of unity, i.e, $e^{2 \pi \iota k/n}$ for $1 \leq k \leq n$.  
This can be seen by noting that multiplication with a circulant matrix gives a convolution. In the Fourier space, convolutions become multiplication. Hence the product of a circulant matrix with a Fourier mode yields a multiple of that Fourier mode, which by definition is an eigenvector. 

The eigenvalues can be then computed as
\[
\omega_k + \omega_k^{-1} = 2 \cos(2 \pi k/n) \quad \text{for} \quad 0 \leq k \leq n-1\,.
\]

Computing $\phi(\omega) \phi(\omega)^\top$ is straightforward. The $(u,v)$-th entry would be just $\omega^{(u+v-2)\mod n}$. Recall that 
\[
\eps_{u \rightarrow v} \leq {\alpha \over \sigma^2} \sum_{i=1}^T  {W_{uv}^i \over i}\,.
\]

For ease of calculation, let us assume that $a=b={1\over 3}$. Then $W = {A_R \over 3} + {\mathbb I \over 3}$, where $A_R$ is a binary matrix with $(i,j)$-th entry $1$ only when $|i -j|=1$. Then the eigenvalues of $W$ are given by $2\cos(2 \pi k/n) +1 \over 3$ for $0 \leq k \leq n-1$. Let $a = (u+v-2)\mod n$. Hence, 

\begin{align*}
\epsilon_{u \to v} 
    &\leq {\alpha \over n\sigma^2} \sum_{t=1}^T \paren{ {e^{2\pi \iota a/n} \over t}  + \sum_{k=1}^{n-1}{1 \over t}\paren{2\cos(2 \pi k/n) +1 \over 3 }^t e^{2\pi \iota a k/n} } \\
    &\leq   {\alpha \over n\sigma^2}\sum_{t=1}^T  \paren{ {e^{2\pi \iota a/n} \over t}  + \sum_{k=1}^{n-1}{1 \over t}\paren{4\cos^2(\pi k/n) - 1 \over 3 }^t e^{2\pi \iota a k/n} } \\
    & = {\alpha  \over n \sigma^2} \sum_{t=1}^T {1 \over t}\cos\left({2\pi a/n} \right)  + {\alpha \over n\sigma^2}\sum_{t=1}^T  \sum_{k=1}^{n-1}{1 \over t}\paren{4\cos^2(\pi k/n) - 1 \over 3 }^t \cos \paren{2\pi a k/n} \\
    & \leq {\alpha \log(T) \over n \alpha^2} + {\alpha \over n\sigma^2}  \sum_{k=1}^{n-1} \sum_{t=1}^\infty {1 \over t} \paren{4\cos^2(\pi k/n) - 1 \over 3 }^t \cos \paren{2\pi  (u+v-2)k \over n} \\
    & = {\alpha \log(T) \over n \sigma^2} - {\alpha \over n\sigma^2}  \sum_{k=1}^{n-1}  \ln \paren{1 - {4\cos^2(\pi k/n) - 1 \over 3} } \cos \paren{2\pi  ak \over n} \\
    &\leq {\alpha \log(T) \over n \sigma^2} +  {2\alpha \over n\sigma^2} \sum_{k=1}^{n-1} \cos \paren{2\pi  ak \over n} \ln \paren{   {3\csc(\pi k/n) \over 2 }  }  ,
\end{align*}
where $\csc$ is the cosecant function. The second equality follows from the fact that $W_{u,v}^i$ is a real number, so the imaginary part is identically zero.%

In the previous bound, we give self-loops the same probability as other edges. The main reason to give self-loop a non-zero weight is to ensure irreducibility and aperiodicity of the Markov chain. The same effect can be achieved by giving any non-negligible weight to the self-loops. In particular, we can consider the following adjacency matrix:
\[
\widehat A_R = (1-\selfloop)A_R + \selfloop \mathbb I
\]
for some $\selfloop >0$. 
Then, the eigenvalues would be $(1-\selfloop)(\omega_k+\omega^{-1}) + \selfloop $. The adjacency matrix is still a circulant matrix. As a result, the eigenvectors still remain the same. Furthermore,
\[
(\omega_k + \omega_k^{-1})^t \cos\paren{2\pi ak \over n} = 2 \cos^t\paren{2\pi k \over n} \cos\paren{2\pi ak \over n} = \cos^{t-1}\paren{2\pi k \over n} \cos\paren{2\pi (a+1)k \over n}. 
\]

Let $\selfloop = {1 \over T^2}$. Then 
$$\widehat A^t_{uv} \leq (1-\selfloop) A^t_{uv}$$
for $t \geq 2$. Therefore, for $a = (u+v-2) \mod n$:
\begin{align*}
    \epsilon_{u \to v} 
    &\leq {\alpha(1-\selfloop) \over n \sigma^2} A_{uv} +  {\alpha(1-\selfloop) \over n \sigma^2}  \sum_{t=2}^T A^t_{uv} \\
    &\leq  {\alpha \over n \sigma^2} A_{uv} +  {\alpha(1-\selfloop) \over n \sigma^2}  \sum_{t=2}^T \sum_{k=1}^n (\omega_k + \omega_k^{-1})^t \cos\paren{2 \pi a k \over n} \\
    & = {\alpha \over n \sigma^2} A_R[u,v] +  {\alpha(1-\selfloop) \over n \sigma^2}  \sum_{t=2}^T \sum_{k=1}^n \cos^{t-1}\paren{2\pi k \over n} \cos\paren{2\pi (a+1)k \over n}\,.
\end{align*}

If $|u-v|=1$, then we have 
\[
\epsilon_{u \to v} \leq {\alpha \over n \sigma^2} +  {\alpha(1-\selfloop) \over n \sigma^2}  \sum_{t=2}^T \sum_{k=1}^n \cos^{t-1}\paren{2\pi k \over n} \cos\paren{2\pi (a+1)k \over n}\,.
\]
otherwise, we have 
\[
\epsilon_{u \to v} \leq   {\alpha(1-\selfloop) \over n \sigma^2}  \sum_{t=2}^T \sum_{k=1}^n \cos^{t-1}\paren{2\pi k \over n} \cos\paren{2\pi (a+1)k \over n}\,.
\]

\paragraph{Star graph.}
We set the vertex set of a simple star graph  as 
\[
V = \{1, 2, \cdots , n \}
\]
with the node $1$ being the central node. This gives the  edge set 
\[
E = \{(1,i), 2 \leq i \leq n \}.
\]

\begin{lemma}
    The eigenvalues  of the Laplacian of the star graph are 
    \[
    \begin{pmatrix} 
    0 & 1 & \cdots & 1 & n
    \end{pmatrix}
    \]
    and the eigenvectors are $\delta_i- \delta_{i+1}$ for eigenvalues $1$ and $2 \leq i \leq n-1$. The eigenvector corresponding to eigenvalue $n$ is computed in the proof.
\end{lemma}
\begin{proof}
    Let $\one$ denote the all one vector. Then by the definition of Lapalcian, $L \one = 0$ and eigenvalue $0$ corresponds to the eigenvector $\one$. Now, the trace of the Laplacian is just the sum of its eigenvalues. Therefore,  
    \[
    \mathsf{Tr}(L)= 2n-2.
    \]
    Let $v$ be the eigenvector for eigenvalue $n$. Then we know that 
    \[
    v \bot \mathsf{Span} \{\one, e_2 - e_3, e_3 - e_4, \cdots, e_{n-1} - e_{n} \}.
    \]
    This implies that 
    \[
    n-1 + v[1] = 0,
    \]
    or that $v = \begin{pmatrix}
        -(n-1) &  1 & 1 & \cdots & 1 
    \end{pmatrix}^\top$. 
\end{proof}

We can perform the spectral decomposition of the adjacency matrix of a star graph, but noting that the adjacency matrix of a star graph is  
\[
A = \begin{pmatrix}
    0 & 1 & 1 & \cdots & 1 \\
    1 & 0 & 0 & \cdots & 0 \\
    1 & 0 & 0 & \cdots & 0 \\
    \vdots & \vdots & \vdots & \ddots & \vdots \\
    1 & 0 & 0 & \cdots & 0 
\end{pmatrix}
\]
it is easy to compute the coordinates of any higher power of $A$. In particular,  if $p$ is an even power, then $(u,v)$-th coordinate of $A_S^p$ is 
\[
A^p_{uv} = \begin{cases}
    (n-1)^{p/2} & u = v = 1\\
    0 & u = 1 \text{ or } v =1 \text{ and } u \neq v \\
    (n-1)^{p/2-1} & \text{otherwise}
\end{cases}
\]
If $p$ is an odd power, then $(u,v)$-th coordinate of $A^p$ is  
\[
A^p_{uv} = \begin{cases}
    (n-1)^{(p-1)/2} & u = 1 \text{ or } v =1\\
    0 & \text{otherwise}
\end{cases}
\]

Since $W = {A \over (n-1)}$ for a normalization constant $(n-1)$ to make $W$ doubly stochastic, we have  
\begin{align*}
\eps_{u \rightarrow v} &\leq  \sum_{p=1}^T A_{uv}^p {\alpha \over \sigma^2 p (n-1)^p} \\
    & = {\alpha \over \sigma^2} \paren{ \sum_{p \text{ is odd}}^T {A_{uv}^p \over p(n-1)^p} + \sum_{p \text{ is even}}^T {A_{uv}^p \over p(n-1)^p} }.
\end{align*}

We now compute the privacy loss for each case. Since we only care about the privacy loss when $u \neq v$, we consider the following two cases:
\begin{enumerate}

    \item $u=1$ or $v=1$ and $u \neq v$. In this case, using \Cref{prop:oddevensum}, we have 
    \begin{align*}
        \eps_{u \rightarrow v} &\leq {\alpha \over \sigma^2}  \sum_{p \text{ is odd}}^T {(\sqrt{n-1})^{p-1} \over p(n-1)^p}  \leq {\alpha \over 2\sigma^2\sqrt{n-1}} {\ln\paren{ \frac{\sqrt{n-1}+1}{\sqrt{n-1}-1}}}\,.
    \end{align*}  

    \item $u \neq 1$ and $u \neq v$. In this case, using \Cref{prop:oddevensum},  we have the following 
    \begin{align*}
        \eps_{u \rightarrow v} &\leq {\alpha \over \sigma^2}  \sum_{p \text{ is even}}^T {({n-1})^{p/2-1} \over p(n-1)^p} \\
        & =  {\alpha \over \sigma^2 (n-1)}  \sum_{p \text{ is even}}^T {({n-1})^{p/2} \over p(n-1)^p} \\
        &\leq -{\alpha \over \sigma^2 (n-1)}\ln \paren{ 1 - \frac{1}{n-1} }\,.
    \end{align*}  
\end{enumerate}

The above calculation is when the graph is simple. To make the Markov chain aperiodic, as before, we add self-loops with a small weight on the self-loop. For example, we consider the following adjacency matrix:
\[
\widehat A = (1-\selfloop)A + \selfloop \mathbb I\,.
\]

We pick $\selfloop= {1 \over T^2}$, so that 
\[
\widehat A^p_{uv} \leq (1-\selfloop) A^p_{uv}
\]
for $u \neq v$ and $p \leq T$. Then if $p$ is an even power, then $(u,v)$-th coordinate of $\widehat A^p$ is 
\[
\widehat A^p_{uv} \leq \begin{cases}
    (1-\selfloop)(n-1)^{p/2} & u = v = 1\\
    0 & u = 1 \text{ or } v =1 \text{ and } u \neq v \\
    (1-\selfloop)(n-1)^{p/2-1} & \text{otherwise}
\end{cases}
\]
If $p$ is an odd power, then $(u,v)$-th coordinate of $\widehat A^p$ is  
\[
\widehat A^p_{uv} \leq \begin{cases}
    (1-\selfloop)(n-1)^{(p-1)/2} & u = 1 \text{ or } v =1\\
    0 & \text{otherwise}
\end{cases}
\]

Now again we have 
\begin{align*}
\eps_{u \rightarrow v} &\leq  \sum_{p=1}^T \widehat A_{uv}^p {\alpha \over \sigma^2 p (n-1)^p} \\
    & = {\alpha \over \sigma^2} \paren{ \sum_{p \text{ is odd}}^T {\widehat A_{uv}^p \over p(n-1)^p} + \sum_{p \text{ is even}}^T {\widehat A_{uv}^p \over p(n-1)^p} }.
\end{align*}

Using the same calculation as before, we have for all $u \neq v$,
\[
\epsilon_{u \to v} \leq \begin{cases}
    {\alpha(1-\selfloop) \over 2\sigma^2\sqrt{n-1}} {\ln\paren{\sqrt{n-1}+1 \over \sqrt{n-1}-1}} & u=1 \text{ or } v=1 \text{ and } u \neq v \\
    - {\alpha (1 -\kappa) \over \sigma^2 (n-1)}\ln \paren{ 1 - {1 \over n-1} } & u \neq 1
\end{cases}.
\]

In particular, this means that the privacy loss for the apex node is the most. In contrast, the nodes on the arms have approximately $\sqrt{n}$ more privacy than the apex node, which is what we expect: the apex node is the only one communicating with every other node in the graph.

\end{document}